\documentclass{article}

\usepackage{microtype}
\usepackage{graphicx}
\usepackage{subfigure}
\usepackage{booktabs} % for professional tables

\usepackage{hyperref}

\usepackage[accepted]{icml2018}

\usepackage{standalone}
\usepackage{chngcntr}
\usepackage{amssymb}
\usepackage{amsmath}
\usepackage{amsthm}
\usepackage{color}
\usepackage{wrapfig}
\usepackage{mathabx}

\usepackage{bbm}
\usepackage{scalerel}
\usepackage{stackengine}
\usepackage{mathtools}
\usepackage[shortlabels]{enumitem}
\usepackage{pgfplots}
\usepackage{adjustbox}
\usepackage{soul}
\usepackage{tikz,tikz-3dplot}
\usetikzlibrary{positioning}
\usetikzlibrary{arrows}
\usetikzlibrary{calc}
\usetikzlibrary{patterns}
\tdplotsetmaincoords{70}{165}

\usetikzlibrary{arrows}
\usepackage{scalerel}
\usetikzlibrary{scopes,matrix,positioning}
\usetikzlibrary{shapes,snakes}
\usepackage{pgfplots}

\usepackage{xr}

\graphicspath{{../fig/}}
\makeatletter
\def\input@path{{../fig/}}
\makeatother
\newcommand{\tp }{{\scriptscriptstyle\mathsf{T}}}
\newcommand{\real}{\mathbb{R}}

\newcommand{\cP}{\mathcal{P}}
\newcommand{\cV}{\mathcal{V}}

\newcommand{\UF}{\operatorname{UF}}

\DeclareMathOperator{\PL}{Pol}
\DeclareMathOperator{\RL}{Rat}

\newcommand{\R}{\mathcal{N}}
\newcommand\pR{\mathcal{N}_c}

\newtheorem{theorem}{Theorem}[section]
\newtheorem{lemma}[theorem]{Lemma}
\newtheorem{proposition}[theorem]{Proposition}

\newtheorem{definition}{Definition}[section]
\newtheorem{corollary}[theorem]{Corollary}
\newtheorem*{question}{Question}
\newtheorem*{theorem*}{Theorem}
\newtheorem*{proposition*}{Proposition}

\renewcommand\binom[2]{\stackMath\mathop{%
\scaleleftright[1.5ex]{(}{\stackanchor[1.8ex]{#1}{#2}}{)}}}

\newcommand\restr[2]{{% we make the whole thing an ordinary symbol
  \left.\kern-\nulldelimiterspace % automatically resize the bar with \right
  #1 % the function
  \right|_{#2} % this is the delimiter
  }}

\pgfplotsset{soldot/.style={color=blue,only marks,mark=*}} \pgfplotsset{holdot/.style={color=blue,fill=white,only marks,mark=*}}
\definecolor{cof}{RGB}{219,144,71}
\definecolor{pur}{RGB}{186,146,162}
\definecolor{greeo}{RGB}{91,173,69}
\definecolor{greet}{RGB}{52,111,72}
\definecolor{lav}{RGB}{147,141,210}
\definecolor{sepia}{rgb}{0.44, 0.26, 0.08}
\definecolor{skobeloff}{rgb}{0.0, 0.48, 0.45}
\definecolor{antiquebrass}{rgb}{0.8, 0.58, 0.46}

\newtoggle{includeSuppl}
\toggletrue{includeSuppl}  % uncomment this line if include supplementary material
\newcommand{\supplFileName}{supp} % !!!!! use the latest supplementary file here !!!

\iftoggle{includeSuppl}
{}
{\externaldocument{\supplFileName}}
\icmltitlerunning{Tropical Geometry of Deep Neural Networks}

\begin{document}

\twocolumn[
\icmltitle{Tropical Geometry of Deep Neural Networks}



\icmlsetsymbol{equal}{*}

\begin{icmlauthorlist}
\icmlauthor{Liwen Zhang}{cs}
\icmlauthor{Gregory Naitzat}{stats}
\icmlauthor{Lek-Heng Lim}{stats,cam}
\end{icmlauthorlist}

\icmlaffiliation{cs}{Department of Computer Science, University of Chicago, Chicago, IL}
\icmlaffiliation{stats}{Department of Statistics, University of Chicago, Chicago, IL}
\icmlaffiliation{cam}{Computational and Applied Mathematics Initiative, University of Chicago, Chicago, IL}

\icmlcorrespondingauthor{Lek-Heng Lim}{lekheng@galton.uchicago.edu}
%\icmlcorrespondingauthor{Eee Pppp}{ep@eden.co.uk}

\icmlkeywords{Neural network, tropical geometry}

\vskip 0.3in
]



\printAffiliationsAndNotice{}  % leave blank if no need to mention equal contribution

\begin{abstract}
We establish, for the first time, connections between feedforward neural networks  with ReLU activation and tropical geometry --- we show that the family of such neural networks is equivalent to the family of tropical rational maps.
Among other things, we deduce that feedforward ReLU neural networks with one hidden layer can be characterized by zonotopes, which serve as building blocks for deeper networks;
we relate decision boundaries of such neural networks  to tropical hypersurfaces, a major object of study in tropical geometry; and we prove that linear regions of such neural networks correspond to vertices of polytopes associated with tropical rational functions.
An insight from our tropical formulation is that a deeper network is exponentially more expressive than a shallow network.
\end{abstract}

\section{Introduction}\label{sec:introduction}
Deep neural networks have recently received much limelight for their enormous success in a variety of applications across many different areas of artificial intelligence,  computer vision, speech recognition, and natural language processing \cite{lecun2015deep, hinton2012deep, krizhevsky2012imagenet, bahdanau2014neural, kalchbrenner2013recurrent}. Nevertheless, it is also well-known that our  theoretical understanding of their efficacy remains incomplete.

There have been several attempts to analyze deep neural networks from different perspectives.  
Notably, earlier studies have suggested that a deep architecture could use parameters more efficiently
and requires exponentially fewer parameters to express certain families of functions than a shallow architecture \cite{delalleau2011shallow, bengio2011expressive, montufar2014number, eldan2016power, poole2016exponential, arora2018understanding}.
Recent work \cite{zhang2016understanding} showed that several successful neural networks possess a high representation power and can easily shatter random data. However, they also generalize well to data unseen during training stage, suggesting that such networks may have some implicit regularization.
Traditional measures of complexity such as VC-dimension and Rademacher complexity fail to explain this phenomenon.
Understanding this implicit regularization that begets the generalization power of deep neural networks  remains a challenge.

The goal of our work is to establish connections between neural network and tropical geometry in the hope that they will shed light on the workings of deep neural networks.
Tropical geometry is a new area in algebraic geometry that has seen an explosive growth in the recent decade but remains relatively obscure outside pure mathematics.
We will focus on feedforward neural networks with rectified linear units (ReLU) and show that they are analogues of \emph{rational functions}, i.e., ratios of two multivariate polynomials $f,g$ in variables $x_1,\dots,x_d$,
\[
\frac{f(x_1,\dots, x_d)}{g(x_1,\dots, x_d)},
\]
in \emph{tropical algebra}. 
For standard and trigonometric polynomials, it is known that \emph{rational approximation} --- approximating a target function by a ratio of two polynomials instead of a single polynomial --- vastly improves the quality of approximation without increasing the degree. 
This gives our analogue: An ReLU neural network is the tropical ratio of two tropical polynomials, i.e., a tropical rational function. 
More precisely, if we view a neural network as a function $\nu : \mathbb{R}^d \to \mathbb{R}^p$, $x = (x_1,\dots,x_d) \mapsto (\nu_1(x),\dots,\nu_p(x))$, then each $\nu$ is a tropical rational map, i.e., each $\nu_i$ is a tropical rational function. In fact, we will show that:
\begin{quote}
\emph{the family of functions represented by feedforward neural networks with rectified linear units and integer weights is exactly the family of tropical rational maps}.
\end{quote}
It immediately follows that there is a \emph{semifield} structure on this family of functions.
More importantly, this establishes a bridge between neural networks\footnote{Henceforth  a ``neural network'' will always mean a feedforward neural network with ReLU activation.} and tropical geometry that allows us to view neural networks as well-studied tropical geometric objects.
This insight allows us to closely relate boundaries between linear regions of a neural network  to tropical hypersurfaces and thereby facilitate studies of decision boundaries of neural networks in classification problems as tropical hypersurfaces.
Furthermore, the number of linear regions, which captures the  complexity of a neural network \citep{montufar2014number, RaghuPKGS17, arora2018understanding}, can be bounded by the number of vertices of the polytopes associated with the neural network's tropical rational representation.
Lastly, a neural network with one hidden layer can be completely characterized by zonotopes, which serve as building blocks for deeper networks.

In Sections~\ref{sec:trop} and \ref{sec:hyper} we introduce  basic tropical algebra and tropical algebraic geometry of relevance to us. 
We state our assumptions  precisely in Section~\ref{sec:neural} and establish the connection between tropical geometry and multilayer neural networks in Section~\ref{sec:tropical-view}.
We analyze neural networks with tropical tools  in Section~\ref{sec:tgnn}, proving that a deeper neural network is exponentially more expressive than a shallow network --- though our objective is not so much to perform state-of-the-art analysis  but to demonstrate that tropical algebraic geometry can provide useful insights. All proofs are deferred to Section~\ref{sec:suppl-proofs} of the supplement.

\section{Tropical algebra}\label{sec:trop}

Roughly speaking, tropical algebraic geometry is an analogue of classical algebraic geometry over $\mathbb{C}$, the field of complex numbers, but where one replaces $\mathbb{C}$ by a  semifield\footnote{A semifield is a field sans the existence of additive inverses.}
called the tropical semiring, to be defined below.
We give a brief review of tropical algebra and introduce some relevant notations.
See \cite{itenberg2009tropical, maclagan2015introduction} for an in-depth treatment.

The most fundamental component of tropical algebraic geometry is the \emph{tropical semiring} $\mathbb{T} \coloneqq \big( \real \cup \{-\infty \}, \oplus, \odot \big)$.
The two operations $\oplus$ and $\odot$, called \emph{tropical addition} and \emph{tropical multiplication} respectively, are defined as follows.
\begin{definition}
For $x,y \in \real$, their \emph{tropical sum} is $x \oplus y \coloneqq \max \{x, y \}$;
their \emph{tropical product} is $ x \odot y \coloneqq x+y$;
the \emph{tropical quotient} of $x$ over $y$ is $x \oslash y \coloneqq x -y$.
\end{definition}
For any $x \in \real$, we have $-\infty \oplus x =  0 \odot x = x$ and $-\infty \odot x = -\infty$.
Thus $-\infty$ is the tropical additive identity and $0$ is the tropical multiplicative identity.
Furthermore, these operations satisfy the usual laws of arithmetic: associativity, commutativity, and distributivity.
The set $\real \cup \{-\infty \}$ is therefore a semiring under the operations $\oplus$ and  $\odot$. While it is not a ring (lacks additive inverse), one may nonetheless generalize many algebraic objects (e.g., matrices, polynomials, tensors, etc) and notions (e.g., rank, determinant, degree, etc) over the tropical semiring --- the study of these, in a nutshell, constitutes the subject of tropical algebra.

Let $\mathbb{N}=\{n\in \mathbb{Z} : n \geq 0 \}$. For an integer $a \in \mathbb{N}$, raising $x \in \real$ to the $a$th power is the same as multiplying $x$ to itself $a$ times. When standard multiplication is replaced by tropical multiplication, this gives us \emph{tropical power}:
\[
x^{\odot a} \coloneqq x \odot \dots \odot x = a \cdot x,
\]
where the last $\cdot$ denotes standard product of real numbers; it is extended to $ \real \cup \{-\infty \}$ by defining, for any $a \in \mathbb{N}$,
\[
-\infty^{\odot a} \coloneqq \begin{cases} -\infty & \text{if}\; a  > 0,\\ 0 & \text{if}\; a  = 0. \end{cases}
\]
A tropical semiring, while not a field, possesses one quality of a field: Every $x \in \mathbb{R}$ has a tropical multiplicative inverse given by its standard additive inverse, i.e., $x^{\odot (-1)} \coloneqq - x$. Though not reflected in its name, $\mathbb{T}$ is in fact a \emph{semifield}.

One may therefore also raise $x \in \mathbb{R}$ to a negative power $a \in \mathbb{Z}$ by raising its tropical multiplicative inverse $-x$ to the positive power $-a$, i.e.,
$x^{\odot a} = (-x)^{\odot (-a)}$.
As is the case in standard real arithmetic, the tropical additive inverse $-\infty$ does not have a tropical multiplicative inverse and $-\infty^{\odot a}$ is undefined for $a <0$. For notational simplicity, we will henceforth write $x^a$ instead of $x^{\odot a}$ for tropical power when there is no cause for confusion. Other algebraic rules of tropical power may be derived from definition; see Section~\ref{sec:suppl-trop-alg} in the supplement.

We are now in a position to define tropical polynomials and tropical rational functions. In the following, $x$ and $x_i$ will denote variables (i.e., indeterminates).
\begin{definition}\label{def:trop_mono}
A \emph{tropical monomial} in $d$ variables $x_1,\dots,x_d$ is an expression of the form
\begin{align*}
c \odot x_1^{a_1} \odot x_2^{a_2} \odot \dots \odot x_d^{a_d}
\end{align*}
where $c \in \real \cup \{-\infty\}$ and $a_1,\dots, a_d \in \mathbb{N}$. As a convenient shorthand, we will also write a tropical monomial in multiindex notation as $c x^\alpha$
where $\alpha=(a_1,\dots, a_d ) \in \mathbb{N}^d$ and $x =(x_1, \dots, x_d)$.
Note that $x^\alpha = 0 \odot x^\alpha$ as $0$ is the tropical multiplicative identity.
\end{definition}

\begin{definition}\label{def:trop_poly}
Following notations above, a \emph{tropical polynomial} $f(x) = f(x_1, \dots, x_d)$ is a finite tropical sum of tropical monomials
\[
f(x)=c_1 x^{\alpha_1} \oplus \dots \oplus c_r x^{\alpha_r},
\]
where $\alpha_i=(a_{i1}, \dots, a_{id}) \in \mathbb{N}^d$ and $c_i \in  \real \cup \{-\infty\}$, $i=1,\dots,r$.
We will assume that a monomial of a given multiindex appears at most once in the sum,  i.e., $\alpha_{i} \neq \alpha_{j}$ for any $i \neq j$.
\end{definition}

\begin{definition}\label{def:trop_rat}
Following notations above, a \emph{tropical rational function} is a standard difference, or, equivalently, a tropical quotient of two tropical polynomials $f(x)$ and $g(x)$:
\[
f(x) - g(x) = f(x) \oslash g(x).
\]
We will denote a tropical rational function by  $ f\oslash g $, where $f$ and $g$ are understood to be tropical polynomial functions.
\end{definition}
It is routine to verify that the set of tropical polynomials $\mathbb{T}[x_1,\dots,x_d]$ forms a  semiring under the standard extension of $\oplus$ and $\odot$ to tropical polynomials,  and likewise the set of tropical rational functions $\mathbb{T}(x_1,\dots,x_d)$ forms  a semifield. We regard a tropical polynomial $f = f \oslash 0$ as a special case
of a tropical rational function and thus $\mathbb{T}[x_1,\dots,x_d] \subseteq \mathbb{T}(x_1,\dots,x_d)$. Henceforth any result stated for a tropical rational function would implicitly also hold for a tropical polynomial.

A $d$-variate tropical polynomial $f(x)$ defines a function $f : \mathbb{R}^d \to \mathbb{R}$ that is a \emph{convex function} in the usual sense as taking max and sum of convex functions preserve convexity \cite{Boyd}. As such, a tropical rational function  $f \oslash g :  \mathbb{R}^d \to \mathbb{R} $ is a \emph{DC function} or \emph{difference-convex function} \cite{hartman1959functions, dca}.

We will need a notion of vector-valued tropical polynomials and tropical rational functions.
\begin{definition}\label{def:tropmorph}
$F : \mathbb{R}^d \to \mathbb{R}^p$, $x = (x_1,\dots,x_d)  \mapsto (f_1(x), \dots, f_p (x))$, is called a \emph{tropical polynomial map} if each $f_i : \mathbb{R}^d \to \mathbb{R}$ is a tropical polynomial, $i=1,\dots,p$, and a \emph{tropical rational map} if $f_1,\dots,f_p$ are tropical rational functions.  We will denote the set of tropical polynomial maps by $\PL(d, p)$ and the set of tropical rational maps by $\RL(d,p)$. So $\PL(d,1) = \mathbb{T}[x_1,\dots,x_d]$ and $\RL(d,1) = \mathbb{T}(x_1,\dots,x_d)$.
\end{definition}

\section{Tropical hypersurfaces}\label{sec:hyper}

There are tropical analogues of many notions in classical algebraic geometry  \cite{itenberg2009tropical, maclagan2015introduction}, among which are \emph{tropical hypersurfaces},  tropical analogues of algebraic curves in classical algebraic geometry.
Tropical hypersurfaces are a principal object of interest in tropical geometry and will prove very useful in our approach towards neural networks. Intuitively, the tropical hypersurface of a tropical polynomial $f$ is
the set of points $x$ where $f$ is not linear at $x$.
\begin{definition}
\label{def:trophype}
The \emph{tropical hypersurface} of  a tropical polynomial $f(x)= c_1 x^{\alpha_1} \oplus \dots \oplus c_r x^{\alpha_r}$ is
\begin{multline*}
\mathcal{T}(f) \coloneqq \big\{ x \in \real^d : 
c_i x^{\alpha_i}= c_j x^{\alpha_j}=f(x) \; \\
\text{for some} \;  \alpha_i \neq \alpha_j \big\}.
\end{multline*}
i.e., the set of points $x$ at which the value of $f$ at $x$ is attained by two or more monomials in $f$.
\end{definition}
A tropical hypersurface divides the domain of $f$ into convex cells on each of which $f$ is linear.
These cells are convex polyhedrons, i.e., defined by linear inequalities with integer coefficients: $\{x \in \mathbb{R}^d : Ax \le b \}$ for $A \in \mathbb{Z}^{m \times d}$ and $b\in \mathbb{R}^m$.
For example,  the cell where a tropical monomial $c_j x^{\alpha_j}$ attains its maximum is $\{ x \in \real^d : c_j +  \alpha_j^\tp x  \geq c_i +  \alpha_i^\tp x\; \text{for all} \; i \neq j  \}$.
Tropical hypersurfaces of polynomials in two variables (i.e., in $\real^2$) are called \emph{tropical curves}.

Just like standard multivariate polynomials, every tropical polynomial comes with an associated \emph{Newton polygon}.
\begin{definition}
The \emph{Newton polygon} of a tropical polynomial $f(x)=c_1x^{\alpha_1}\oplus \dots \oplus c_rx^{\alpha_r}$ is the convex hull of $\alpha_1,\dots,\alpha_r \in \mathbb{N}^d$, regarded as points in $\mathbb{R}^d$,
\begin{align*}
\Delta (f) \coloneqq \operatorname{Conv} \bigl\{ \alpha_i  \in \mathbb{R}^d : c_i \neq -\infty , \, i=1,\dots,r \bigr\}.
\end{align*}
\end{definition}

\begin{figure}
\centering
	\begin{tikzpicture}[thick, scale=0.3]
	
	\def\ra{4} % dot radius 
	\def\scale{4} % scaling
	\def\a{\scale * 1}
	
	% the first part of tropical curve
	%
	\begin{scope}
	\begin{scope}[shift={(0,0)}]
	\draw[blue]  (-\a, 0) -- (\a, 0);
	\draw[blue]  (\a, 0) -- (\a, -\a);
	\draw[blue]  (\a, 0) -- (\a * 2, \a);
	\end{scope}
	
	% the second part of tropical curve
	%
	\begin{scope}[shift={(-\a,0)}]
	\draw[blue] (0, \a) -- (\a, \a);
	\draw[blue]  (\a, \a) -- (\a, -\a);
	\draw[blue]  (\a, \a) -- (\a * 2, 2 * \a);
	\end{scope}
	\end{scope}	
	\end{tikzpicture}
	~~
	\begin{tikzpicture}[thick, scale=0.3]
	
	% arrow type  +head or tail
	%
	\def\arrA{latex-}
	\def\arrB{-latex}
	
	\def\ra{4.5} % dot radious
	\def\scale{4} % scaling
	\def\a{\scale * 1}
	
	% tropical curve
	%
	\begin{scope}
	\begin{scope}[shift={(\a,0)}]
	\draw[dashed, color =blue, line width=0.4mm, \arrA]  (-\a, 0) -- (\a, 0);
	\draw[dashed, color =blue, line width=0.4mm, \arrB]  (\a, 0) -- (\a, -\a);
	\draw(\a, 0)[fill = blue, blue ]  circle[color = blue, radius = \ra pt];
	\draw[dashed, color =blue, line width=0.4mm,  \arrB]  (\a, 0) -- (\a * 2, \a);
	\end{scope}
	
	\begin{scope}[shift={(0,0)}]
	\draw[dashed,  color =blue, line width=0.4mm, \arrA]  (0, \a) -- (\a, \a);
	\draw[dashed,  color =blue, line width=0.4mm, \arrB]  (\a, \a) -- (\a, -\a);
	\draw(\a, \a) [fill = blue, blue ]   circle[radius = \ra pt];
	\draw[dashed,  color =blue, line width=0.4mm, \arrB]  (\a, \a) -- (\a * 2, 2 * \a);
	\end{scope}
	\end{scope}
	
	\def\ra{3.5} 
	
	% newton polygon
	%
	\begin{scope}[shift={(2.5,-2)}]
	\def\scaleb{9}
	\def\b{\scaleb * 1}
	\fill (0, 0)  circle[radius = \ra pt];
	\draw (0, 0)  -- (\b, 0);
	\fill (\b, 0) circle[radius = \ra pt];
	\draw (\b, 0) -- (0, \b);
	\fill (0, \b) circle[radius = \ra pt];
	\draw (0, \b) -- (0, 0);
	
	\draw (0, 0.5 * \b)  -- (0.5 * \b,  0.5 * \b);
	\fill (0.5 * \b,  0.5 * \b) circle[radius = \ra pt];
	\draw (0.5 * \b,  0.5 * \b)  -- (0.5 * \b, 0);
	\fill (0, 0.5 * \b) circle[radius = \ra pt];
	\fill (0.5 * \b, 0) circle[radius = \ra pt];
	\end{scope}
	
	\end{tikzpicture}	
\caption{$1 \odot x_1^2  \oplus 1 \odot x_2^2 \oplus 2 \odot x_1 x_2 \oplus 2\odot x_1 \oplus 2\odot x_2 \oplus 2$. Left: Tropical curve.  Right: Dual subdivision of Newton polygon and tropical curve.
\label{fig:tropical-curve-eg2}}
\end{figure}
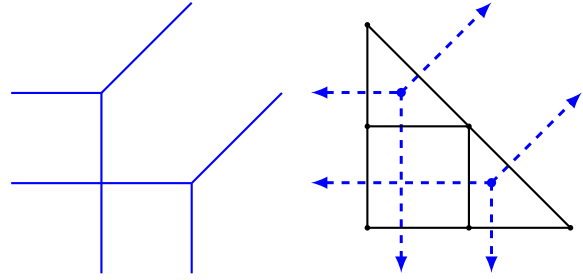

A tropical polynomial $f$ determines a dual subdivision of $\Delta (f)$, constructed as follows.
First, lift each $\alpha_i$ from $\real^d$ into $\real^{d+1}$ by appending $c_i$ as the last coordinate. Denote the convex hull of the lifted $\alpha_1,\dots,\alpha_r$ as
\begin{equation}\label{eq:polytope-F}
\cP(f)\coloneqq \operatorname{Conv}\{ (\alpha_i, c_i) \in \real^{d} \times \real : i=1,\dots,r\} .
\end{equation}
Next let $\UF\bigl(\cP(f)\bigr)$ denote the collection of upper faces in $\cP(f)$ and
$\pi : \real^{d} \times \real \rightarrow \real^{d}$ be the projection that drops the last coordinate.
The dual subdivision determined by  $f$ is then
\begin{align*}
\delta(f) \coloneqq \bigl\{ \pi(p) \subset \real^d : p \in \UF\bigl(\cP(f)\bigr) \bigr\}.
\end{align*}
$\delta( f )$ forms a polyhedral complex with support $\Delta(f)$.
By \citep[Proposition~3.1.6]{maclagan2015introduction}, the tropical hypersurface $\mathcal{T}(f)$ is the
$(d-1)$-skeleton of the polyhedral complex dual to $\delta(f)$. 
This means that each vertex in $\delta(f)$ corresponds to one ``cell'' in $\real^d$ where the function $f$ is linear.
Thus, the number of vertices in $\mathcal{P}(f)$ provides an upper bound on the number of linear regions of $f$.

Figure~\ref{fig:tropical-curve-eg2} shows the Newton polygon and dual subdivision for the tropical polynomial
$f(x_1, x_2) = 1 \odot x_1^2  \oplus 1 \odot x_2^2 \oplus 2 \odot x_1 x_2 \oplus 2\odot x_1 \oplus 2\odot x_2 \oplus 2$. 
Figure~\ref{fig:polytope-subdivision} shows how we may find the dual subdivision for this tropical polynomial by following the aforementioned procedures; with step-by-step details given in Section~\ref{sec:suppl-tropical-eg}.

Tropical polynomials and tropical rational functions are clearly piecewise linear functions. As such a tropical rational map is a piecewise linear map and the notion of \emph{linear region} applies.
\begin{definition}\label{def:lr}
A \emph{linear region} of $F\in \RL(d,m)$ is a maximal connected subset of the domain on which $F$ is linear.
The number of linear regions of $F$ is denoted $\R(F)$.
\end{definition}
Note that a tropical \emph{polynomial} map $F \in \PL(d, m)$ has convex linear regions but a tropical \emph{rational} map $F \in \RL(d,n)$ generally has nonconvex linear regions. In Section~\ref{sec:bounds}, we will use $\R(F)$ as a  measure of complexity for an $F  \in \RL(d,n)$ given by a neural network.

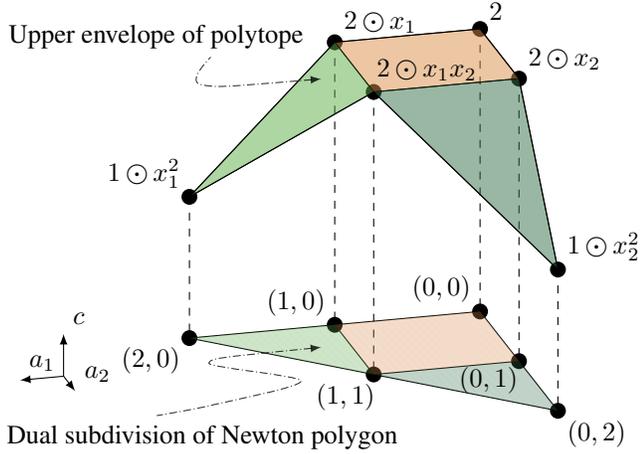
\begin{figure}
\centering
\begin{tikzpicture}[scale=2, tdplot_main_coords]

% arrow type + heasd or tails
%
\def\arrA{latex-}
\def\arrB{-latex}

\coordinate (O) at (0,0,0);

% Upper envelop vertices
%
\def\a{1}
\def\b{1}
\def\ra{1.5}
\coordinate (A) at (0, 0, 2 * \b);
\coordinate (B) at (1 * \a, 0, 2 * \b);
\coordinate (C) at (2 * \a, 0, 1 * \b);
\coordinate (D) at (1 * \a, 1 * \a, 2 * \b);
\coordinate (E) at (0, 1 * \a, 2 * \b);
\coordinate (F) at (0, 2 * \a, 1 * \b);

% Newton polygon vertices
%
\coordinate (a) at (0, 0, 0);
\coordinate (b) at (1 * \a, 0, 0);
\coordinate (c) at (2 * \a, 0, 0);
\coordinate (d) at (1 * \a, 1 * \a, 0);
\coordinate (e) at (0, 1 * \a, 0);
\coordinate (f) at (0, 2 * \a, 0);

% small axis plot
%
\begin{scope}[shift={(0,0.5, 0)}]
\def\scalaxis{0.3}
\coordinate (axis_o) at (3 * \a, 0 , 0);
\coordinate (axis_a1) at (3 * \a + \scalaxis, 0 , 0);
\coordinate (axis_a2) at (3 * \a, \scalaxis , 0);
\coordinate (axis_c) at (3 * \a, 0 ,  + \scalaxis);
\end{scope}

\draw[\arrB] (axis_o) -- (axis_c) node [above right] {$c$};
\draw[\arrB] (axis_o) -- (axis_a1) node [above right] {$a_1$};
\draw[\arrB] (axis_o) -- (axis_a2) node [above right] {$a_2$};

% dot over upper envelop + labels
%
\fill (A) circle[radius = \ra pt] node [above right] {$2$};
\fill (B) circle[radius = \ra pt] node [above right] {$2\odot x_1$};
\fill (C) circle[radius = \ra pt] node [above left] {$1\odot x_1^2$};

% D is repeated at the end of the drawing in order for the label to be on top 
%
\fill (D) circle[radius = \ra pt]; 
\fill (E) circle[radius = \ra pt] node [above right] {$2\odot x_2$};
\fill (F) circle[radius = \ra pt] node [above right] {$1\odot x_2^2$};

% dots over Newton polygon (uncomment labels for debugging)
%
\fill (a) circle[radius = \ra pt]; %node [above left] {$a ~2$};
\fill (b) circle[radius = \ra pt]; %node [above left] {$b ~2\odot x_1$};
\fill (c) circle[radius = \ra pt]; %node [above left] {$c ~1\odot x_1^2$};
\fill (d) circle[radius = \ra pt]; %node [above left] {$d ~2\odot x_1x_2$};
\fill (e) circle[radius = \ra pt]; %node [above left] {$e ~2\odot x_2$};
\fill (f) circle[radius = \ra pt]; %node [above left] {$f ~1\odot x_2^2$};

% dashed lines from the upper envelop to Newton polygon
%
\draw[dashed] (A) -- (a);
\draw[dashed] (B) -- (b);
\draw[dashed] (C) -- (c);
\draw[dashed] (D) -- (d);
\draw[dashed] (E) -- (e);
\draw[dashed] (F) -- (f);

% Newton subdevision
%
\draw[] (a) -- (b) -- (c) -- (d) -- (f) -- (e) -- (a);
\draw[] (b) -- (d);
\draw[] (d) -- (e);

\fill[cof,opacity=0.3](a) -- (b) -- (d) -- (e) -- cycle;
\fill[greeo,opacity=0.3](c) -- (b) -- (d) -- cycle;
\fill[greet,opacity=0.3](d) -- (e) -- (f) -- cycle;

% upper envelop faces
%
\draw (A) -- (B) -- (D) -- (E) -- cycle;
\fill[cof,opacity=0.5](A) -- (B) -- (D) -- (E) -- cycle;

\draw (C) -- (B) -- (D) -- cycle;
\fill[greeo,opacity=0.5](C) -- (B) -- (D) -- cycle;

\draw (D) -- (E) -- (F) -- cycle;
\fill[greet,opacity=0.5](D) -- (E) -- (F) -- cycle;

% "walls"
%
\draw (C) -- (D) -- (F);
%\fill[gray,opacity=0.2] (c) -- (C) -- (D) -- (F) -- (f) -- cycle;

\draw (C) -- (B) -- (A);
%\fill[gray,opacity=0.2](c) -- (C) -- (B) -- (A) -- (a) -- cycle;

% legend
%
\tikzstyle{densely dashdotted}=      [dash pattern=on 3pt off 1pt on \the\pgflinewidth off 1pt]
\fill [step=0.1cm, pattern color=blue, pattern = north west lines, opacity=0.15](a) -- (b) -- (c) -- (d) -- (f) -- (e) -- cycle;

\draw[opacity=0.8, densely dashdotted] [\arrA] plot  [smooth] coordinates {(1.2 * \a, 0.4 * \a, 0) (2  * \a, 0.5  * \a, 0)  (1.5  * \a, 1  * \a, 0)  (2.2  * \a, 1.3  * \a, 0) (2.6  * \a, 1.4  * \a, 0)
};

\fill (2.2  * \a, 1.2  * \a, 0) node [below, xshift = -0.4ex, yshift = -1.3ex] {Dual subdivision of Newton polygon};

\draw (A) -- (E) -- (F);
%\fill[gray,opacity=0.4](a) -- (A) -- (E) -- (F) -- (f) -- cycle;

\draw[opacity=0.8, densely dashdotted] [\arrA] plot  [smooth] coordinates {(1.2 * \a, 0.4 * \a, 0.95 * 2 * \a) (2 * \a, 0.3 * \a, 0.95 * 2 * \a)  (1.8 * \a, -0.2 * \a, 0.95 * 2 * \a) 
};
\fill (3.2 * \a, -0.4 * \a, 0.98 * 2 * \a) node [above right] {Upper envelope of polytope};

% labels on the dots
%

\fill (a) node [above left] {$(0, 0)$};
\fill (b) node [above left] {$(1, 0)$};
\fill (c) node [below left] {$(2, 0)$};
\fill (d) node [below left, xshift = 0.9ex] {$(1, 1)$};
\fill (e) node [below left, xshift = 0.7ex, yshift = 0.3ex] {$(0, 1)$};
\fill (f) node [below right] {$(0, 2)$};

\fill (D) node [above right, yshift = 0.2ex, xshift = -0.5ex] {$2\odot x_1x_2$};
\end{tikzpicture}
\caption{$1 \odot x_1^2  \oplus 1 \odot x_2^2 \oplus 2 \odot x_1 x_2 \oplus 2\odot x_1 \oplus 2\odot x_2 \oplus 2$.
The dual subdivision can be obtained by projecting the edges on the upper faces of the polytope.
\label{fig:polytope-subdivision}}
\end{figure}

\subsection{Transformations of tropical polynomials} \label{sec:transform-tropical-poly}

Our analysis of neural networks will require figuring out how the polytope $\mathcal{P}(f)$ transforms under tropical power, sum, and  product. The first is straightforward.
\begin{proposition}
	\label{prop:polytope-exp}
	Let $f$ be a tropical polynomial and let $a \in \mathbb{N}$. Then
	\[
	\cP(f^a) = a \cP(f).
	\]
\end{proposition}
$a \cP(f)=\{ ax : x \in \cP(f) \} \subseteq \real^{d+1}$ is a scaled version of $\cP(f)$ with the same shape but different volume.

To describe the effect of tropical sum and product, we need a few notions from convex geometry. The \emph{Minkowski sum}  of two sets  $P_1$ and $P_2$ in $\real^d$ is the set 
\[
P_1 + P_2 \coloneqq \big\{ x_1 + x_2 \in \real^d :  x_1 \in P_1, x_2 \in P_2 \big\};
\]
and for $\lambda_1, \lambda_2 \geq 0$, their \emph{weighted Minkowski sum} is
\[
\lambda_1 P_1 + \lambda_2 P_2\coloneqq \big\{ \lambda_1 x_1 + \lambda_2 x_2 \in \real^d :  x_1 \in P_1, x_2 \in P_2 \big\}.
\]
Weighted Minkowski sum is clearly commutative and associative and generalizes to more than two sets.
In particular, the Minkowski sum of line segments is called a \emph{zonotope}.

Let $\cV (P)$ denote the set of vertices of a polytope $P$. Clearly, the Minkowski sum of two polytopes is given by the convex hull of the Minkowski sum of their vertex sets, i.e.,
$P_1 + P_2 = \operatorname{Conv} \bigl( \cV(P_1)+\cV(P_2) \bigr)$.
With this observation, the following is immediate.
\begin{proposition}
	\label{prop:polytope-ops}
	Let $f, g \in \PL(d,1) = \mathbb{T}[x_1,\dots,x_d]$ be tropical polynomials. Then
	\begin{align*}
	\cP(f \odot g ) &= \cP(f) + \cP(g), \\
	\cP(f \oplus g) &= \operatorname{Conv} \bigl( \cV ( \cP(f) ) \cup \cV ( \cP(g) ) \bigr).
	\label{eq:polytope-ops}
	\end{align*}
\end{proposition}

We reproduce below part of \citep[Theorem 2.1.10]{gritzmann1993minkowski} and derive a corollary for bounding the number of verticies on the upper faces of a zonotope.
% which we  will later use for counting vertices in polytopes.
\begin{theorem}[Gritzmann--Sturmfels]
	\label{thm:minkowski-face-bound}
Let $P_1, \dots, P_k$ be polytopes in $\real^d$ and let $m$ denote the total number of nonparallel edges of $P_1, \dots, P_k$. 
Then the number of vertices of $P_1 + \dots + P_k$ does not exceed
	\begin{align*}
	2 \sum_{j=0}^{d-1} \binom{m-1}{j}.
	\end{align*}
The upper bound is attained if all $P_i$'s are zonotopes and all their generating line segments are in general positions.
\end{theorem}

\begin{corollary}
\label{cor:num-vert-on-uf}
Let $P \subseteq \real^{d+1}$ be a zonotope generated by $m$ line segments $P_1, \dots, P_m$.
Let $\pi : \real^d \times \real \to \real^d$ be the projection.
Suppose $P$ satisfies:
\begin{enumerate}[\upshape (i), topsep=0ex, itemsep=0ex]
\item\label{it:seg} the generating line segments are in general positions; 
\item\label{it:ver} the set of projected vertices $\{ \pi(v) : v \in \mathcal{V} (P) \} \subseteq \real^d$ are in general positions.
\end{enumerate}
Then $P$ has
\[
\sum_{j=0}^{d} \binom{m}{j}
\]
vertices on its upper faces.
If either \ref{it:seg} or \ref{it:ver} is violated, then this becomes an upper bound.
\end{corollary}
As we mentioned, linear regions of a tropical polynomial $f$ correspond to vertices on $\UF\bigl(\cP(f)\bigr)$ and the corollary will be useful for bounding the number of linear regions.

\section{Neural networks}\label{sec:neural}
While we expect our readership to be familiar with feedforward neural networks, we will nevertheless use this short section to define them,
primarily for the purpose of fixing notations and specifying the assumptions that we retain throughout this article. We restrict our attention to fully connected feedforward neural networks.

Viewed abstractly, an $L$-layer feedforward neural network is a map $\nu : \mathbb{R}^d \to \mathbb{R}^p$
given by a composition of functions
\[
\nu = \sigma^{(L)} \circ \rho^{(L)} \circ \sigma^{(L - 1)} \circ \rho^{(L-1)} \dots \circ \sigma^{(1)} \circ \rho^{(1)}.
\]
The \emph{preactivation} functions $\rho^{(1)},\dots, \rho^{(L)}$ are affine transformations to be determined and the \emph{activation} functions $\sigma^{(1)},\dots,\sigma^{(L)}$ are chosen and fixed in advanced. 

We denote the \emph{width}, i.e., the number of nodes, of the $l$th layer by $n_l$, $l=1, \cdots, L-1$. We set $n_0 \coloneqq d$ and $n_L \coloneqq p$, respectively the dimensions of the input and output of the network.
The output from the $l$th layer will be denoted by
\[
\nu^{(l)} \coloneqq
\sigma^{(l)} \circ \rho^{(l)} \circ \sigma^{(l- 1)} \circ \rho^{(l-1)} \dots \circ \sigma^{(1)} \circ \rho^{(1)},
\]
i.e., it is a map $\nu^{(l)} : \real^d \to \real^{n_l}$.
For convenience, we assume $\nu^{(0)}(x) \coloneqq x$.

The affine function $\rho^{(l)}:\real^{n_{l-1}} \to \real^{n_l}$ is given by a \emph{weight} matrix $A^{(l)} \in \mathbb{Z}^{{n_l} \times n_{l-1}}$ and a \emph{bias} vector $b^{(l)} \in \real^{n_{l}}$:
\[
\rho^{(l)}(\nu^{(l-1)})\coloneqq  A^{(l)}\nu^{(l-1)} +b^{(l)}.
\]
The $(i,j)$th coordinate of  $A^{(l)}$ will be denoted $a_{ij}^{(l)}$ and
the $i$th coordinate of $b^{(l)}$ by $b_{i}^{(l)}$.
Collectively they form the \emph{parameters} of the $l$th layer.

For a vector input $x\in \mathbb{R}^{n_l}$, $\sigma^{(l)}(x)$ is understood to be in coordinatewise sense; so $\sigma : \mathbb{R}^{n_l} \to \mathbb{R}^{n_l}$. 
We assume the final output of a neural network $\nu(x)$ is fed into a \emph{score function}  $s : \mathbb{R}^{p} \to \mathbb{R}^m$  that is application specific. 
When used as an $m$-category classifier, $s$ may be chosen, for example, to be a soft-max or sigmoidal function.
The score function is quite often regarded as the last layer of a neural network but this is purely a matter of convenience and we will not assume this.
We will make the following mild assumptions on the architecture of our feedforward neural networks and explain next why they are indeed mild:
\begin{enumerate}[\upshape (a), topsep=0ex, itemsep=0ex]
	\item\label{ass1} the weight matrices $A^{(1)},\dots,A^{(L)}$  are integer-valued;
	\item\label{ass2} the bias vectors $b^{(1)},\dots,b^{(L)}$ are real-valued;
    \item\label{ass3} the activation functions $\sigma^{(1)},\dots,\sigma^{(L)}$ take the form 
    \[
    \sigma^{(l)}(x) \coloneqq  \max\{x, t^{(l)}\},
    \]
where $t^{(l)} \in (\real \cup \{-\infty \})^{n_l}$ is called a \emph{threshold} vector.
\end{enumerate}
Henceforth all neural networks in our subsequent discussions will be assumed to satisfy \ref{ass1}--\ref{ass3}.

\ref{ass2} is completely general but there is also no loss of generality in \ref{ass1},
i.e., in   restricting the weights  $A^{(1)},\dots,A^{(L)}$ from real matrices to integer matrices, as:
\begin{itemize}[topsep = 0ex, itemsep=0ex]
\item real weights can be approximated arbitrarily closely by rational weights;
\item one may then `clear denominators' in these rational weights by multiplying them by the least common multiple of their denominators to obtain integer weights;
\item keeping in mind that scaling all weights and biases by the same positive constant has no bearing on the workings of a neural network.
\end{itemize}

The activation function in  \ref{ass3} includes both ReLU activation ($t^{(l)} = 0$) and identity map ($t^{(l)} = -\infty$) as special cases. Aside from ReLU, our tropical framework will apply to piecewise linear activations such as leaky ReLU and absolute value, and with some extra effort, may be extended to max pooling, maxout nets, etc. But it does not, for example, apply to activations such as hyperbolic tangent and sigmoid.

In this work, we view an ReLU network as the simplest and most canonical model of a neural network, from 
which other variants that are more effective at specific tasks may be derived. Given that we seek general theoretical insights and not specific practical efficacy, it makes sense to limit ourselves to this simplest case.
Moreover, ReLU networks already embody some of the most important elements (and mysteries) common to a wider range of neural networks (e.g., universal  approximation, exponential expressiveness); they
work well in practice and are often the go-to choice for feedforward networks. We are also not alone in limiting our discussions to ReLU networks \cite{montufar2014number, arora2018understanding}.

\section{Tropical algebra of neural networks}
\label{sec:tropical-view}

We now describe our tropical formulation of a multilayer feedforward neural network satisfying \ref{ass1}--\ref{ass3}.

A multilayer feedforward neural network is generally nonconvex,
whereas a tropical polynomial is always convex. 
Since most nonconvex functions are a difference of two convex functions \cite{hartman1959functions}, a reasonable guess is that a feedforward neural network  is the difference of two tropical polynomials, i.e., a tropical rational function. This is indeed the case, as we will see from the following.

Consider the output from the first layer in neural network
\begin{align*}
\nu(x) &=\max \{A x+b, \, t \},
\end{align*}
where $A \in \mathbb{Z}^{p \times d}, b \in \mathbb{R}^{p}$, and $t \in (\real \cup \{-\infty\})^p$. We will decompose $A$ as a difference of two nonnegative integer-valued matrices, $A = A_+ - A_-$ with $A_+, A_- \in \mathbb{N}^{p \times d}$; 
e.g., in the standard way with entries
\[
a^{+}_{ij} \coloneqq \max \{ a_{ij}, 0\}, \qquad
a^{-}_{ij} \coloneqq \max \{-a_{ij}, 0\}
\]
respectively. 
Since
\[
\max \{ Ax+b , t \} = \max \{ A_+ x + b, \, A_- x + t\} - A_- x,
\]
we see that every coordinate of one-layer neural network is a difference of two tropical polynomials.
For networks with more layers, we apply this decomposition recursively to obtain the following result.
\begin{proposition}\label{prop:representation2}
Let $A \in \mathbb{Z}^{m \times n}$, $b \in \mathbb{R}^m$ be the parameters of the $(l+1)$th layer, and let  $t \in (\real \cup \{-\infty\})^m$ be the threshold vector in the $(l+1)$th layer.
If the nodes of the $l$th layer are given by tropical rational functions,
\[
\nu^{(l)}(x) = F^{(l)}(x) \oslash G^{(l)}(x) =  F^{(l)}(x) - G^{(l)}(x),
\]
i.e., each coordinate of $F^{(l)}$ and $G^{(l)}$ is a  tropical polynomial in $x$,
then the outputs of the preactivation and of the $(l+1)$th layer are given by tropical rational functions
\begin{align*}
\rho^{(l+1)}\circ \nu^{(l)}(x) &= H^{(l+1)}(x) -  G^{(l+1)}(x), \\
\nu^{(l+1)}(x) = \sigma \circ\rho^{(l+1)}\circ \nu^{(l)}(x) &=  F^{(l+1)}(x) - G^{(l+1)}(x)
\end{align*}
respectively, where
\begin{align*}
F^{(l+1)}(x) &= \max \bigl\{ H^{(l+1)}(x),\, G^{(l+1)}(x) + t \bigr\}, \\                    
G^{(l+1)}(x) &= A_+ G^{(l)}(x) + A_- F^{(l)}(x), \\
H^{(l+1)}(x) &= A_+ F^{(l)}(x)+A_- G^{(l)}(x) + b.
\end{align*}
We will write $f^{(l)}_i$, $g^{(l)}_i$ and $h^{(l)}_i$ for the $i$th coordinate of $F^{(l)}$, $G^{(l)}$ and $H^{(l)}$ respectively. In tropical arithmetic, the recurrence above takes the form
\begin{equation} \label{eq:tropical-recursion}
\begin{aligned}
f^{(l+1)}_i &= h_i^{(l+1)} \oplus (g_i^{(l+1)} \odot t_i),\\
g^{(l+1)}_i &= \biggl[ \bigodot_{j=1}^{n} ( f^{(l)}_{j} )^{a_{ij}^-} \biggr]
             \odot \biggl[ \bigodot_{j=1}^{n} ( g^{(l)}_{j} )^{a_{ij}^+} \biggr],\\                     
h^{(l+1)}_i &= \biggl[ \bigodot_{j=1}^{n} ( f^{(l)}_{j} )^{a_{ij}^+} \biggr]
             \odot \biggl[ \bigodot_{j=1}^{n} ( g^{(l)}_{j})^{a_{ij}^-} \biggr]
             \odot b_i.
\end{aligned}
\end{equation}
\end{proposition}
Repeated applications of Proposition~\ref{prop:representation2} yield the following.
\begin{theorem}[Tropical characterization of neural networks]\label{thm:trop_char}
A  feedforward neural network under assumptions \ref{ass1}--\ref{ass3} is a function $\nu : \mathbb{R}^d \to \mathbb{R}^p$ whose coordinates are tropical rational functions of the input, i.e.,
\[
\nu(x) = F(x) \oslash G(x) = F(x) - G(x)
\]
where $F$ and $G$ are tropical polynomial maps. Thus $\nu$ is a tropical rational map.
\end{theorem}

Note that the tropical rational functions above have real coefficients, not integer coefficients. The integer weights $A^{(l)} \in \mathbb{Z}^{n_l \times n_{l-1}}$ have gone into the powers of tropical monomials in $f$ and $g$, which is why we require our weights to be integer-valued, although as we have explained, this requirement imposes little loss of generality.

By setting $t^{(1)} =\dots =t^{(L-1)} =  0$ and $t^{(L)} = -\infty$, we obtain the following corollary.
\begin{corollary}\label{cor:FFNNReLU-is-RTF}
Let $\nu : \mathbb{R}^d \to \mathbb{R}$ be an ReLU activated feedforward neural network with integer weights and linear output. Then $\nu$ is a tropical rational function. 
\end{corollary}
A more remarkable fact is the converse of Corollary~\ref{cor:FFNNReLU-is-RTF}.
	\begin{theorem}[Equivalence of neural networks and tropical rational functions]\label{thm:RTF-is-FFNNReLU}\hfill
\begin{enumerate}[\upshape (i), topsep=0ex, itemsep=0ex]
\item Let $\nu : \mathbb{R}^d \to \mathbb{R}$. Then $\nu$ is a tropical rational function if and only if $\nu$ is a feedforward neural network satisfying assumptions \ref{ass1}--\ref{ass3}.

\item\label{layerbound} A tropical rational function $f \oslash g$ can be represented as an $L$-layer neural network, with
\[
L \le \max \{ \lceil \log_2 r_f \rceil , \, \lceil \log_2 r_g \rceil \}  + 2,
\]
where $r_f$ and $r_g$ are the number of monomials in the tropical polynomials $f$ and $g$ respectively.
\end{enumerate}
\end{theorem}
We would like to acknowledge the precedence of \citep[Theorem~2.1]{arora2018understanding}, which demonstrates the equivalence between ReLU-activated $L$-layer neural networks with \emph{real} weights and $d$-variate continuous piecewise functions with \emph{real} coefficients, where  $L \le \lceil \log_2(d+1) \rceil + 1$. 

By construction, a tropical rational function is a continuous piecewise linear function. The continuity of a piecewise linear function automatically implies that  each of the pieces on which it is  linear is a  polyhedral region.  
As we saw in Section~\ref{sec:hyper},  a tropical polynomial $f : \real^d \to \real$ gives a tropical hypersurface that divides $\real^d$ into \emph{convex} polyhedral regions
defined by linear inequalities with integer coefficients: $\{ x \in \mathbb{R}^d : Ax \le b\}$ with $A \in \mathbb{Z}^{m \times d}$ and $b \in \mathbb{R}^m$.
A tropical rational function $f \oslash g : \real^d \to \real$ must also be a continuous piecewise linear function and divide $\real^d$ into polyhedral regions on each of which $f \oslash g$ is linear, although these regions are \emph{nonconvex} in general.
We will show the converse --- any continuous piecewise linear function with integer coefficients is a tropical rational function.
\begin{proposition}
\label{prop:CPWL-is-RTF}
Let $\nu : \mathbb{R}^d \to \mathbb{R}$. Then $\nu$ is a continuous piecewise linear function with integer coefficients if and only if $\nu$ is a tropical rational function.
\end{proposition}

Corollary~\ref{cor:FFNNReLU-is-RTF}, Theorem~\ref{thm:RTF-is-FFNNReLU}, and Proposition~\ref{prop:CPWL-is-RTF} collectively imply the equivalence of
\begin{enumerate}[\upshape (i), topsep=0ex, itemsep=0ex]
\item tropical rational functions,
\item continuous piecewise linear functions with integer coefficients,
\item neural networks satisfying assumptions \ref{ass1}--\ref{ass3}.
\end{enumerate}
An immediate advantage of this characterization is that the set of tropical rational functions $\mathbb{T}(x_1,\dots,x_d)$ has a semifield structure as we pointed out in Section~\ref{sec:trop}, a fact that we have implicitly used in the proof of Proposition~\ref{prop:CPWL-is-RTF}. However, what is more important is not the algebra but the \emph{algebraic geometry} that arises from our tropical characterization. We will use tropical algebraic geometry to illuminate our understanding of neural networks in the next section.

The need to stay within tropical algebraic geometry is the reason we did not go for a simpler and more general characterization (that does not require the integer coefficients assumption). A \emph{tropical signomial} takes the form 
\[
\varphi(x) = \bigoplus_{i=1}^m b_i \bigodot_{j=1}^n x_j^{a_{ij}},
\]
where $a_{ij} \in \mathbb{R}$ and $b_i \in \real \cup \{-\infty\}$. Note that $a_{ij}$ is not required to be integer-valued nor nonnegative. A \emph{tropical rational signomial} is a tropical quotient $\varphi \oslash \psi$ of two tropical signomials $\varphi, \psi$.  A \emph{tropical rational signomial map} is a function $\nu = (\nu_1,\dots, \nu_p) : \mathbb{R}^d \to \mathbb{R}^p$ where each $\nu_i : \mathbb{R}^d \to \mathbb{R}$ is a tropical rational signomial $\nu_i = \varphi_i \oslash \psi_i$. The same argument we used to establish Theorem~\ref{thm:trop_char} gives us the following.
\begin{proposition}\label{prop:tropsig}
Every feedforward neural network with ReLU activation is  a tropical rational signomial map.
\end{proposition}
Nevertheless tropical signomials fall outside the realm of tropical algebraic geometry and we do not use Proposition~\ref{prop:tropsig} in the rest of this article.

\section{Tropical geometry of neural networks}\label{sec:tgnn}

Section~\ref{sec:tropical-view} defines neural networks via tropical algebra, a perspective that allows us to study them via tropical algebraic geometry. 
We will show that the decision boundary of a neural network is a subset of a tropical hypersurface of a corresponding tropical polynomial (Section~\ref{sec:boundary}). 
We will see that, in an appropriate sense, zonotopes form the geometric building blocks for neural networks (Section~\ref{sec:zono}). We then prove that the geometry of the function represented by a neural network grows vastly more complex as its  number of layers  increases (Section~\ref{sec:bounds}).

\subsection{Decision boundaries of a neural network}\label{sec:boundary}

We will use tropical geometry and insights from Section~\ref{sec:tropical-view} to study decision boundaries of neural networks, focusing on the case of two-category classification for clarity. 
As explained in Section~\ref{sec:neural}, a neural network $\nu :\mathbb{R}^d \to \mathbb{R}^p$ together with a choice of score function $s :\mathbb{R}^p \to \mathbb{R}$ give us a classifier.
If the output value $s(\nu(x))$ exceeds some decision threshold $c$, then the neural network predicts $x$ is from one class (e.g., $x$ is a \textsc{cat} image), and otherwise $x$ is from the other category (e.g., a \textsc{dog} image). 
The input space is thereby partitioned into two disjoint subsets by the \emph{decision boundary} $\mathcal{B} \coloneqq \{x\in \mathbb{R}^d :\nu(x) = s^{-1}(c)\}$.
Connected regions with value above the threshold and connected regions with value below the threshold will be called the \emph{positive regions} and \emph{negative regions} respectively.

We provide bounds on the number of positive and negative regions and show that there is a tropical polynomial whose tropical hypersurface contains the decision boundary.
\begin{proposition}[Tropical geometry of decision boundary]\label{prop:db}
Let $\nu : \mathbb{R}^d \to \mathbb{R}$ be an $L$-layer neural network satisfying assumptions \ref{ass1}--\ref{ass3} with $t^{(L)} = -\infty$. Let the score function $s : \mathbb{R} \to \mathbb{R}$ be  injective with decision threshold $c$ in its range. If $\nu = f \oslash g$ where $f$ and $g$ are tropical polynomials, then 
\begin{enumerate}[\upshape (i), topsep=0ex, itemsep=0ex]
\item \label{prop:db-item2} its decision boundary $\mathcal{B} =\{x\in \mathbb{R}^d : \nu(x) = s^{-1}(c)\}$ divides $\real^d$ into at most $\R(f)$ connected positive regions and at most $\R(g)$ connected negative regions;
\item its decision boundary  is contained in the tropical hypersurface of the tropical polynomial $s^{-1}(c)\odot g(x) \oplus f(x) = \max \{ f(x), \, g(x)+ s^{-1}(c) \}$,  i.e.,
\begin{equation}\label{eq:decbd}
\mathcal{B}  \subseteq  \mathcal{T} ( s^{-1}(c)\odot g \oplus f ) .
\end{equation}
\end{enumerate}
\end{proposition}
The function $s^{-1}(c)\odot g \oplus f$ is not necessarily linear on every positive or negative region and so its tropical hypersurface $\mathcal{T} (s^{-1}(c)\odot g \oplus f)$ may further divide a positive or negative region derived from $\mathcal{B}$ into multiple linear regions.
Hence the ``$\subseteq$'' in \eqref{eq:decbd} cannot in general be replaced by ``$=$''.

\subsection{Zonotopes as geometric building blocks of neural networks}\label{sec:zono}

From Section~\ref{sec:hyper}, we know that the number of regions a tropical hypersurface $\mathcal{T}(f)$ divides the space into  equals the number of vertices in the dual subdivision of the Newton polygon associated with the tropical polynomial $f$. This allows us to bound the number of linear regions of a neural network by bounding the number of vertices in the dual subdivision of the Newton polygon.

We start by examining how geometry changes from one layer to the next  in a neural network, more precisely:
\begin{question}
How are the tropical hypersurfaces of the tropical polynomials in the $(l+1)$th layer of a neural network related to those  in the $l$th layer?
\end{question}
The recurrent relation \eqref{eq:tropical-recursion} describes how the tropical polynomials occurring in the $(l+1)$th layer are obtained from those
in the $l$th layer, namely, via three operations: tropical sum, tropical product, and tropical powers. Recall  that a tropical hypersurface of a tropical polynomial is dual to the dual subdivision of the Newton polytope of the tropical polynomial, which is given by the projection of the upper faces on the polytopes defined by \eqref{eq:polytope-F}.
Hence the question boils down to how these three operations transform the polytopes, which is  addressed in Propositions~\ref{prop:polytope-exp} and \ref{prop:polytope-ops}. We follow notations in  Proposition~\ref{prop:representation2} for the next result.
\begin{lemma}
\label{lemma:polytopes}
Let $f^{(l)}_i$, $g^{(l)}_i$, $h^{(l)}_i$ be the tropical polynomials produced by the $i$th node in the $l$th layer of a neural network, i.e., they are defined by \eqref{eq:tropical-recursion}. 
Then  $\cP \bigl(f^{(l)}_i\bigr)$, $\cP \bigl(g^{(l)}_i\bigr)$, $\cP \bigl(h^{(l)}_i\bigr)$ are subsets of $\real^{d+1}$ given as follows:
\begin{enumerate}[\upshape (i), topsep=0ex, itemsep=0ex]
\item $\cP \bigl( g^{(1)}_i \bigr)$ and  $\cP \bigl( h^{(1)}_i \bigr)$ are points.
\item $\cP \bigl( f^{(1)}_i \bigr)$ is a line segment.
\item\label{zono} $\cP \bigl( g^{(2)}_i \bigr)$ and $\cP \bigl( h^{(2)}_i \bigr)$ are zonotopes.
\item\label{conv} For $l \geq 1$,
\[
\cP \bigl( f^{(l)}_i \bigr) = 
\operatorname{Conv}\bigl[\cP \bigl( g^{(l)}_i  \odot t_i^{(l)} \bigr) \cup \cP \bigl( h^{(l)}_i\bigr) \bigr]
\]
if $t_i^{(l)}  \in \mathbb{R}$, and $\cP \bigl( f^{(l)}_i \bigr) = \cP \bigl( h^{(l)}_i \bigr)$ if $ t_i^{(l)} = -\infty$.
\item For $l \geq 1$,  $\cP \bigl( g^{(l+1)}_i \bigr)$ and $\cP \bigl( h^{(l+1)}_i \bigr)$ are  weighted Minkowski sums,
\begin{align*}
\cP \bigl( g^{(l+1)}_i \bigr) &= \sum_{j=1}^{n_l} a_{ij}^{-} \cP \bigl( f^{(l)}_{j} \bigr)
+ \sum_{j=1}^{n_l} a_{ij}^{+} \cP \bigl( g^{(l)}_{j} \bigr),\\
\cP \bigl( h^{(l+1)}_i \bigr) &= \sum_{j=1}^{n_l} a_{ij}^{+} \cP \bigl( f^{(l)}_{j} \bigr)
                                    + \sum_{j=1}^{n_l} a_{ij}^{-} \cP \bigl( g^{(l)}_{j} \bigr) \\
&\qquad + \{ b_i e \},
\end{align*}
where $a_{ij}$, $b_i$ are entries of the weight matrix $A^{(l+1)} \in \mathbb{Z}^{n_{l+1} \times n_l}$ and  bias vector $b^{(l+1)} \in \mathbb{R}^{n_{l+1}}$, and $e \coloneqq (0,\dots,0,1) \in \mathbb{R}^{d+1}$.
\end{enumerate}
\end{lemma}
A conclusion of Lemma~\ref{lemma:polytopes} is that zonotopes are the building blocks in the tropical geometry of neural networks.
Zonotopes are studied extensively in convex geometry and, among other things, are intimately related to hyperplane arrangements \cite{greene1983interpretation, guibas2003zonotopes, mcmullen1971zonotopes, holtz2011zonotopal}. Lemma~\ref{lemma:polytopes} connects neural networks to this extensive body of work but its full implication remains to be explored.  
In Section~\ref{sec:suppl-poly-of-nn-eg} of the supplement, we show how one may build these polytopes for a two-layer neural network.

\subsection{Geometric complexity of deep neural networks}\label{sec:bounds}

We apply the tools in Section~\ref{sec:hyper} to study the complexity of a neural network, showing that a deep network is much more expressive than a shallow one.
Our measure of complexity is geometric: we will follow \cite{montufar2014number, RaghuPKGS17} and use the number of linear regions of a piecewise linear function $\nu : \mathbb{R}^d \to \mathbb{R}^p$ to measure the complexity of $\nu$.

We would like to emphasize that our upper bound below does not improve on that obtained in \cite{RaghuPKGS17} --- in fact, our version is more restrictive given that it applies only to neural networks satisfying \ref{ass1}--\ref{ass3}. Nevertheless our goal here is to demonstrate how tropical geometry may be used to derive the same bound.
\begin{theorem}\label{thm:main-bound}
Let $\nu : \mathbb{R}^d \to \mathbb{R}$ be an $L$-layer real-valued feedforward neural network satisfying  \ref{ass1}--\ref{ass3}. Let $t^{(L)} = -\infty$ and  $n_{l} \geq d$ for all $l=1,\dots,L-1$.  Then $\nu=\nu^{(L)}$ has at most
\[
\prod_{l=1}^{L-1} \sum_{i=0}^{d} \binom{n_l}{i}
\]
linear regions.
In particular, if $d \le n_{1}, \dots, n_{L-1} \le n$, the number of  linear regions of $\nu$ is bounded by $\mathcal{O} \bigl( {n}^{d(L-1)} \bigr)$.
\end{theorem}
\begin{proof}
If $L=2$, this follows directly from Lemma~\ref{lemma:polytopes} and Corollary~\ref{cor:num-vert-on-uf}.
The case of $L \ge 3$ is in Section~\ref{prof:main-bound} in the supplement.
\end{proof}

As was pointed out in \citep{RaghuPKGS17}, this upper bound closely matches the lower bound $\Omega\bigl((n / d)^{(L-1)d} n^d\bigr)$ in \citep[Corollary 5]{montufar2014number} when $ n_{1} = \dots = n_{L-1} = n \ge d$.
Hence we surmise that the number of linear regions of the neural network grows polynomially with the width $n$ and exponentially with the number of layers $L$.

\section{Conclusion}

We argue that feedforward neural networks with rectified linear units are, modulo trivialities, nothing more than tropical rational maps. To understand them we often just need to understand the relevant tropical geometry.

In this article, we took a first step to provide a proof-of-concept: questions regarding decision boundaries, linear regions, how depth affect expressiveness, etc, can be translated into questions involving tropical hypersurfaces, dual subdivision of Newton polygon, polytopes constructed from zonotopes, etc.

As a new branch of algebraic geometry, the novelty of tropical geometry stems from both the algebra and geometry as well as the interplay between them. It has connections to many other areas of mathematics.
Among other things, there is a tropical analogue of linear algebra \citep{butkovivc2010max} and a tropical analogue of convex geometry \citep{gaubert2006max}. We cannot emphasize enough that we have only touched on a small part of this rich subject. We hope that further investigation from this tropical angle might perhaps unravel other mysteries of deep neural networks.

\section*{Acknowledgments}
The authors  thank Ralph Morrison, Yang Qi, Bernd Sturmfels, and the anonymous referees for their very helpful comments.  The work in this article is generously supported by DARPA D15AP00109, NSF IIS 1546413, the Eckhardt Faculty Fund, and  a DARPA Director's Fellowship.

%\bibliography{reference}
\bibliographystyle{icml2018}

\iftoggle{includeSuppl}{
\providecommand{\includeSuppl}{true}
\onecolumn
%%\twocolumn[
\icmltitle{Supplementary Material: Tropical Geometry of Deep Neural Networks}

\appendix
\counterwithin{figure}{section}

\section{Illustration of our neural network}\label{sec:suppl-nn-fig}

Figure~\ref{fig:nn} summarizes the architecture and notations of the feedforward neural network discussed in this paper.

\begin{figure*}[h!]
\centering
	\begin{adjustbox}{width=\textwidth}
\input{figure4_new.tikz}
%\vspace*{-5ex}
\end{adjustbox}
\centering
\caption{General form of an ReLU feedforward neural network $\nu:\real^d \to \real^p$ with $L$ layers.}
\label{fig:nn}
\end{figure*}

\section{Tropical power} \label{sec:suppl-trop-alg}

As in Section~\ref{sec:trop}, we write $x^a = x^{\odot a}$; aside from this slight abuse of notation, $\oplus$ and $\odot$ denote tropical sum and product, $+$ and $\cdot$ denote standard sum and product in all other contexts. Tropical power evidently has the following properties:
\begin{itemize}[topsep=0ex, itemsep = 0ex]
\item For $x,y \in \real$ and $a \in \mathbb{R}$, $a \geq 0$, 
\[
( x \oplus y )^a = x^a \oplus y^a \quad \text{and} \quad ( x \odot y )^a = x^a \odot y^a.
\]
If $a$ is allowed negative values, then we lose the first property. In general $( x \oplus y )^a \neq x^a \oplus y^a $ for $a<0$.

\item For $x \in \real$, 
\[
x^0 = 0.
\]
\item For $x \in \real$ and $a, b \in \mathbb{N}$,
\[
( x^a )^{ b }  = x^{a \cdot b}.
\]
\item For $x \in \real$ and $a, b \in \mathbb{Z}$,
\[
x^a \odot x^{ b } = x^{a + b}.
\]
\item For $x \in \real$ and $a, b \in \mathbb{Z}$,
\[
x^a \oplus x^{ b } = x^a \odot ( x^{a - b} \oplus 0 ) = x^a \odot ( 0 \oplus x^{a - b} ).
\]
\end{itemize}

\section{Examples}

\subsection{Examples of tropical curves and dual subdivision of Newton polygon} \label{sec:suppl-tropical-eg}

Let $f \in \PL(2,1) = \mathbb{T}[x_1, x_2]$, i.e., a bivariate tropical polynomial. 
It follows from our discussions in Section~\ref{sec:hyper} that the tropical hypersurface $\mathcal{T}(f)$ is a planar graph dual to the dual subdivision $\delta(f)$ in the following sense:
\begin{enumerate}[\upshape (i), topsep=0ex, itemsep=0ex]
\item Each two-dimensional face in $\delta(f)$ corresponds to a vertex in $\mathcal{T}(f)$.
\item Each one-dimensional edge of a face in $\delta(f)$ corresponds to an edge in $\mathcal{T}(f)$.  
In particular, an edge from the Newton polygon $\Delta (f)$ corresponds to an unbounded edge in $\mathcal{T}(f)$ while other edges correspond to bounded edges.
\end{enumerate}

Figure~\ref{fig:polytope-subdivision} illustrates how we may find the dual subdivision for the tropical polynomial $f(x_1, x_2) = 1 \odot x_1^2  \oplus 1 \odot x_2^2 \oplus 2 \odot x_1 x_2 \oplus 2\odot x_1 \oplus 2\odot x_2 \oplus 2$.
First, find the convex hull 
\begin{align*}
\mathcal{P}(f)=\operatorname{Conv} & \{  (2,0,1),(0,2,1),(1,1,2), (1,0,2),(0,1,2),(0,0,2) \}.
\end{align*}
Then, by projecting the upper envelope of $\mathcal{P}(f)$ to $\real^2$, we obtain $\delta(f)$, the dual subdivision of the Newton polygon.

\subsection{Polytopes of a two-layer neural network} \label{sec:suppl-poly-of-nn-eg}

We illustrate our discussions in Section~\ref{sec:zono} with a two-layer example. Let $\nu : \mathbb{R}^2 \to \mathbb{R}$ be with $n_0 = 2$ input nodes, $n_1 =5$ nodes in the first layer, and $n_2 = 1$ nodes in the output:
\begin{align*}
y =\nu^{(1)}(x) & = \max \left\{ \begin{bmatrix*}[r]
-1 & 1 \\ 1 & -3 \\ 1 & 2 \\ -4 & 1 \\ 3 & 2
\end{bmatrix*}
\begin{bmatrix}
x_1 \\ x_2
\end{bmatrix}
+ \begin{bmatrix*}[r]
1 \\ -1 \\ 2 \\ 0 \\ -2
\end{bmatrix*}, \, 0 \right\}, \\
\nu^{(2)}(y) & = \max \{ y_1 +  2y_2 + y_3 -y_4 -3y_5, \, 0 \}.
\end{align*}
\begin{wrapfigure}[7]{r}[0ex]{0.42\textwidth}
\centering
\includegraphics[width = 0.41\textwidth]{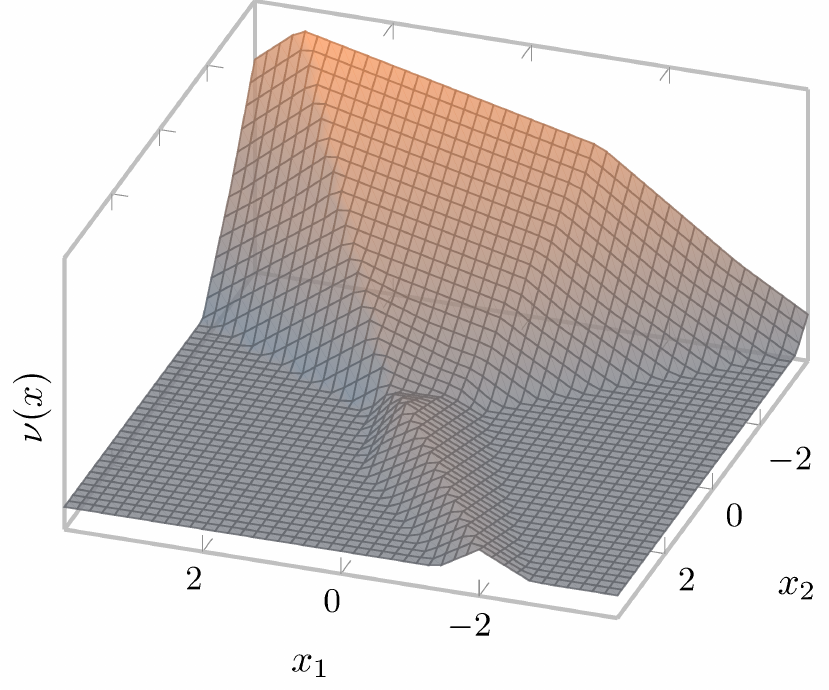}
%  \caption{Birds}
\end{wrapfigure}
We first express $\nu^{(1)}$ and $\nu^{(2)}$ as tropical rational maps,
\[
\nu^{(1)} = F^{(1)} \oslash G^{(1)}, \quad \nu^{(2)} = f^{(2)} \oslash g^{(2)},
\]
where
\begin{align*}
y &\coloneqq F^{(1)}(x) = H^{(1)}(x) \oplus G^{(1)}(x), \\
z &\coloneqq G^{(1)}(x) =
 \begin{bmatrix}
x_1 \\
x_2^3 \\
0 \\
x_1^4 \\
0
\end{bmatrix}, \qquad
H^{(1)}(x) = 
\begin{bmatrix}
1 \odot x_2 \\
(-1) \odot x_1 \\
2 \odot x_1 x_2^2 \\
x_2 \\
(-2) \odot x_1^3 x_2^2
\end{bmatrix},
\end{align*}
and
\begin{align*}
f^{(2)}(x) &= g^{(2)}(x) \oplus h^{(2)}(x), \\
g^{(2)}(x) &= y_4 \odot y_5^3 \odot z_1 \odot z_2^2 \odot z_3 \\
        &= ( x_2 \oplus x_1^4 ) 
           \odot ((-2) \odot x_1^3 x_2^2 \oplus 0 )^3
           \odot x_1 \odot (x_2^3 )^2, \\
h^{(2)}(x) &= y_1 \odot y_2^2 \odot y_3 \odot z_4 \odot z_5^3 \\
        &= ( 1 \odot x_2 \oplus x_1 ) \odot ( (-1)\odot x_1 \oplus x_2^3 )^2 
        \odot ( 2\odot x_1 x_2^2 \oplus 0 ) \odot x_1^4. 
\end{align*}
We will write $F^{(1)} = (f^{(1)}_1,\dots,f^{(1)}_5)$ and likewise for $G^{(1)}$ and $H^{(1)}$.
The monomials occurring in $g_j^{(1)}(x)$ and $h_j^{(1)}(x)$ are all of the form $cx_1^{a_1}x_2^{a_2}$.
Therefore $\mathcal{P}( g^{(1)}_j)$ and $\mathcal{P}( h^{(1)}_j)$, $j =1,\dots,5$, are points in $\real^3$.

Since $F^{(1)}=G^{(1)} \oplus H^{(1)}$, $\mathcal{P}(f^{(1)}_j)$ is a convex hull of two points, and thus a line segment in $\real^{3}$.
The Newton polygons associated with $f^{(1)}_j$, equal to their dual subdivisions in this case, are obtained by projecting these line segments back to the plane spanned by $a_1, a_2$, as shown on the left in Figure~\ref{fig:eg-f1-g2}.

The line segments $\mathcal{P}(f^{(1)}_j)$,  $j=1,\dots,5$, and points $\mathcal{P}(g^{(1)}_j)$,  $j=1,\dots,5$, serve as building blocks for $\mathcal{P}(h^{(2)})$ and $\mathcal{P}(g^{(2)})$, which are constructed as weighted Minkowski sums:
\begin{align*}
\mathcal{P}(h^{(2)}) &= \mathcal{P}(f^{(1)}_4) + 3\mathcal{P}(f^{(1)}_5) + \mathcal{P}(g^{(1)}_1) + 2\mathcal{P}(g^{(1)}_2) + \mathcal{P}(g^{(1)}_3),  \\
\mathcal{P}(g^{(2)}) &= \mathcal{P}(f^{(1)}_1) + 2\mathcal{P}(f^{(1)}_2) + \mathcal{P}(f^{(1)}_3)  + \mathcal{P}(g^{(1)}_4) + 3\mathcal{P}(g^{(1)}_5).
\end{align*}
$\mathcal{P}(g^{(2)})$ and the dual subdivision of its Newton polygon are shown on the right in Figure~\ref{fig:eg-f1-g2}. $\mathcal{P}(h^{(2)})$ and the dual subdivision of its Newton polygon are shown on the left in Figure~\ref{fig:eg-h2-f2}.
$\mathcal{P}(f^{(2)})$ is the convex hull of the union of $\mathcal{P}(g^{(2)})$ and $\mathcal{P}(h^{(2)})$. The dual subdivision of its Newton polygon is obtained by projecting the upper faces of $\mathcal{P}(f^{(2)})$ to the plane spanned by $a_1, a_2$. These are shown on the right in Figure~\ref{fig:eg-h2-f2}.

\begin{figure}
	\centering
		% f(1)
%
\begin{tikzpicture}[thick, scale =0.55, tdplot_main_coords]
	\tikzset{font={\fontsize{7pt}{12}\selectfont}}
	\coordinate (O) at (0,0,0);
	% arrow type + heasd or tails
	%
	\def\arrA{latex-}
	\def\arrB{-latex}	
	
	% Upper envelop vertices
	%
	\def\a{2}
	\def\b{1.5}
	\def\ra{2.5}
	\def\lw{1.5}
	\coordinate (ha) at (0, 1 * \a, 0);
	\coordinate (hb) at (1 * \a, 0, 0);
	\coordinate (hc) at (1 * \a, 2 * \a, 0);
	\coordinate (hd) at (0, 1 * \a, 0);
	\coordinate (he) at (3 * \a, 2 * \a , 0);
	
	\def\shiftg{3}
	\coordinate (hA) at (0, 1 * \a, 1 * \b + \shiftg * \b);
	\coordinate (hB) at (1 * \a, 0, -1 * \b + \shiftg * \b);	
	\coordinate (hC) at (1 * \a, 2 * \a, 2 * \b + \shiftg * \b);
	\coordinate (hD) at (0, 1 * \a,  0 + \shiftg * \b);
	\coordinate (hE) at (3 * \a, 2 * \a , -2 * \b + \shiftg * \b);	
		
	\coordinate (gA) at (1 * \a, 0, \shiftg * \b);
	\coordinate (gB) at (0, 3 * \a, \shiftg * \b);
	\coordinate (gC) at (0 * \a, 0, \shiftg * \b);
	\coordinate (gD) at (4 * \a, 0, \shiftg * \b);
	\coordinate (gE) at (0, 0, \shiftg * \b);
	
	\coordinate (ga) at (1 * \a, 0, 0);
	\coordinate (gb) at (0, 3 * \a, 0);
	\coordinate (gc) at (0, 0, 0);
	\coordinate (gd) at (4 * \a, 0, 0);
	\coordinate (ge) at (0, 0, 0);

	% dashed lines from the upper envelop to Newton polygon
	%
	\draw[dashed, red, opacity=0.4] (gA) -- (ga);
	\draw[dashed, blue, opacity=0.4] (gB) -- (gb);
	\draw[dashed, green, opacity=0.4] (gC) -- (gc);
	\draw[dashed, black, opacity=0.4] (gD) -- (gd);
	\draw[dashed, brown, opacity=0.4] (gE) -- (ge);
	
	\draw[dashed, red, opacity=0.4] (hA) -- (ha);
	\draw[dashed, blue, opacity=0.4] (hB) -- (hb);
	\draw[dashed, green, opacity=0.4] (hC) -- (hc);
	\draw[dashed, black, opacity=0.4] (hD) -- (hd);
	\draw[dashed, brown, opacity=0.4] (hE) -- (he);
	
	\fill ($(gA)!0.2!(hA)$) node [above, yshift = -0.1ex] {${\color{red}f^{(1)}_1}$};
	\fill ($(gB)!0.5!(hB)$) node [below, xshift = -0.75ex] {${\color{blue}f^{(1)}_2}$};
	\fill ($(gC)!0.85!(hC)$) node [right] {${\color{green} f^{(1)}_3}$};
	\fill ($(gD)!0.3!(hD)$) node [below] {${\color{black} f^{(1)}_4}$};
	\fill ($(gE)!0.7!(hE)$) node [below] {${\color{brown}f^{(1)}_5}$};

	\draw[red, opacity=0.4] (ga) -- (ha);
	\draw[blue, opacity=0.4] (gb) -- (hb);
	\draw[green, opacity=0.4] (gc) -- (hc);
	\draw[black, opacity=0.4] (gd) -- (hd);
	\draw[brown, opacity=0.4] (ge) -- (he);

	\draw[line width=\lw, red] (gA) -- (hA);
	\draw[line width=\lw, blue] (gB) -- (hB);
	\draw[line width=\lw, green] (gC) -- (hC);
	\draw[line width=\lw, black] (gD) -- (hD);
	\draw[line width=\lw, brown] (gE) -- (hE);
	
	\begin{scope}		
	\fill (gA) circle[radius = \ra pt]node [left] {$g_1^{(1)}$};
	\fill (gB) circle[radius = \ra pt]node [right] {$g_2^{(1)}$};
	%	\fill (gC) circle[radius = \ra pt]node [right] {$g_3^{(1)}$};
	\fill (gD) circle[radius = \ra pt]node [left] {$g_4^{(1)}$};
	\fill (gE) circle[radius = \ra pt]node [right] {$g_3^{(1)}\backslash g_5^{(1)}$};
	\end{scope}
	
	\begin{scope}		
	\fill (ga) circle[radius = \ra pt]node [left] {$(1,0)$};
	\fill (gb) circle[radius = \ra pt]node [right] {$(0,3)$};
	\fill (gc) circle[radius = \ra pt]node [above, xshift=-0.2ex] {$(0,0)$};
	\fill (gd) circle[radius = \ra pt]node [left] {$(4,0)$};
	\fill (ge) circle[radius = \ra pt];%node [above left] {$(0,0)$};
	\end{scope}	
	
	\begin{scope}
	%\fill (ha) circle[radius = \ra pt] node [right] {$(0,1)$};	
	\fill (hb) circle[radius = \ra pt]; % node [above left] {$(1,0)$};
	\fill (hc) circle[radius = \ra pt]; % node [above left] {$(0,0)$};
	\fill (hd) circle[radius = \ra pt] node [right, xshift = -0.4ex] {$(0,1)$};
	\fill (he) circle[radius = \ra pt] node [left] {$(3,2)$};	
	\end{scope}
	
	\begin{scope}
	\fill (hA) circle[radius = \ra pt] node [above right, yshift = -0.5ex, xshift = -0.5ex] {$h_1^{(1)}$};	
	\fill (hB) circle[radius = \ra pt] node [left] {$h_2^{(1)}$};
	\fill (hC) circle[radius = \ra pt] node [left, yshift=0.7ex, xshift=0.5ex] {$h_3^{(1)}$};
	\fill (hD) circle[radius = \ra pt] node [below right, xshift=-0.4, yshift=0.6ex ] {$h_4^{(1)}$};
	\fill (hE) circle[radius = \ra pt] node [left, yshift = 1ex, xshift = 1ex] {$h_5^{(1)}$};	
	\end{scope}

	% small axis plot
	%
	\begin{scope}[shift={(0,0.5, 0)}]
	\def\scalaxis{1.5}
	\coordinate (axis_o) at (-2*\a, 0 , 0);
	\coordinate (axis_a1) at (-2*\a + \scalaxis, 0 , 0);
	\coordinate (axis_a2) at (-2*\a, \scalaxis , 0);
	\coordinate (axis_c) at (-2*\a, 0 ,  + \scalaxis);
	\end{scope}
	
	\draw[\arrB] (axis_o) -- (axis_c) node [above right] {$c$};
	\draw[\arrB] (axis_o) -- (axis_a1) node [above right] {$a_1$};
	\draw[\arrB] (axis_o) -- (axis_a2) node [above right] {$a_2$};	
	\end{tikzpicture}	\begin{tikzpicture}[thick, scale = 0.5, tdplot_main_coords]
	\tikzset{font={\fontsize{7pt}{12}\selectfont}}
	\coordinate (O) at (0,0,0);
	
	% arrow type + heasd or tails
	%
	\def\arrA{latex-}
	\def\arrB{-latex}	
	
	% Upper envelop vertices
	%
	\def\a{0.4}
	\def\b{1}
	\def\ra{3.5}
	\def\s{7.3}
	\def\lw{1.5}
	\coordinate (A) at (10* \a, 13 * \a, -6 * \b +\s);
	\coordinate (B) at (14 * \a, 12 * \a, -6 * \b +\s);
	\coordinate (C) at (1 * \a, 7 * \a, 0 +\s);
	\coordinate (D) at (5 * \a, 6 * \a, 0 +\s);
	
	\coordinate (a) at (10* \a, 13 * \a, 0);
	\coordinate (b) at (14 * \a, 12 * \a, 0);
	\coordinate (c) at (1 * \a, 7 * \a, 0);
	\coordinate (d) at (5 * \a, 6 * \a, 0);
	
	\begin{scope}
	\fill (b) circle[radius = \ra pt] node [left] {$(14, 12)$};
	\fill (c) circle[radius = \ra pt] node [below, xshift =2ex, yshift=0.5ex] {$(1, 7)$};
	\fill (d) circle[radius = \ra pt] node [right, yshift = 1.2ex, xshift = -0.6ex] {$(5, 6)$};
	\end{scope}
	
	\begin{scope}
	\fill (A) circle[radius = \ra pt] node [below right, xshift = 8.5ex, yshift =-2.5ex] {$(-6) \odot x_1^{10}x_2^{13}$};	
	\fill (B) circle[radius = \ra pt] node [above left] {$(-6)\odot x_1^{14}x_2^{12}$};
	\fill (C) circle[radius = \ra pt] node [right] {$x_1x_2^7$};
	\fill (D) circle[radius = \ra pt] node [left] {$x_1^5x_2^6$};
	\end{scope}	
	
	\draw[] (a) -- (b) -- (d) -- (c) -- cycle;
	\fill [step=0.1cm, pattern color=gray, opacity=0.2](a) -- (b) -- (d) -- (c) -- cycle;
	\fill [step=0.1cm, pattern color=blue, pattern = north west lines, opacity=0.15](a) -- (b) -- (d) -- (c) -- cycle;
	
	\draw[line width=\lw] (A) -- (B) -- (D) -- (C) -- cycle;

	\fill (a) circle[radius = \ra pt] node [below] {$(10, 13)$};	
	
	\draw[dashed] (a) -- (A);
	\draw[dashed] (b) -- (B);
	\draw[dashed] (c) -- (C);
	\draw[dashed] (d) -- (D);
	
		\fill[cof,opacity=0.5](A) -- (B) -- (D) -- (C) -- cycle;
		% small axis plot
	%
	\begin{scope}[shift={(0,0.5, 0)}]
	\def\scalaxis{1.5}
	\coordinate (axis_o) at (-12.5*\a, 0 , 3);
	\coordinate (axis_a1) at (-12.5*\a + \scalaxis, 0 , 3);
	\coordinate (axis_a2) at (-12.5*\a, \scalaxis , 3);
	\coordinate (axis_c) at (-12.5*\a, 0 ,  3+ \scalaxis);
	\end{scope}
	
	\draw[\arrB] (axis_o) -- (axis_c) node [above right] {$c$};
	\draw[\arrB] (axis_o) -- (axis_a1) node [above right] {$a_1$};
	\draw[\arrB] (axis_o) -- (axis_a2) node [above right] {$a_2$};	
	
	\draw[opacity=0.8, densely dashdotted] [\arrA] plot  [smooth] coordinates {(9.8* \a, 13 * \a, -6 * \b +\s)
		 (7* \a, 15 * \a, -7 * \b +\s)  (5* \a, 15 * \a, -6.7 * \b +\s) (3* \a, 15 * \a, -7.3 * \b +\s)
	};

	\end{tikzpicture}
\vspace*{-2ex}
\caption{Left: $\mathcal{P}(F^{(1)})$ and dual subdivision of $F^{(1)}$. Right: $\mathcal{P}(g^{(2)})$ and dual subdivision of $g^{(2)}$. 
In both figures, dual subdivisions have been translated along the $-c$ direction (downwards) and separated from the polytopes for visibility.}
\label{fig:eg-f1-g2}
\end{figure}
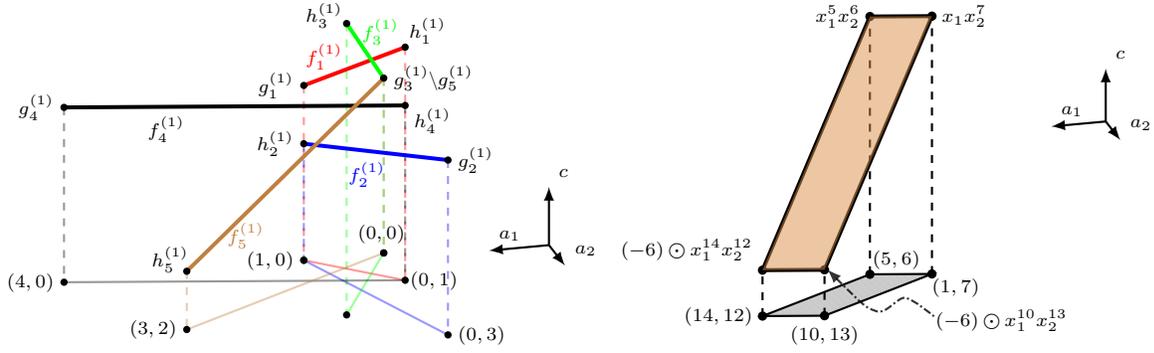
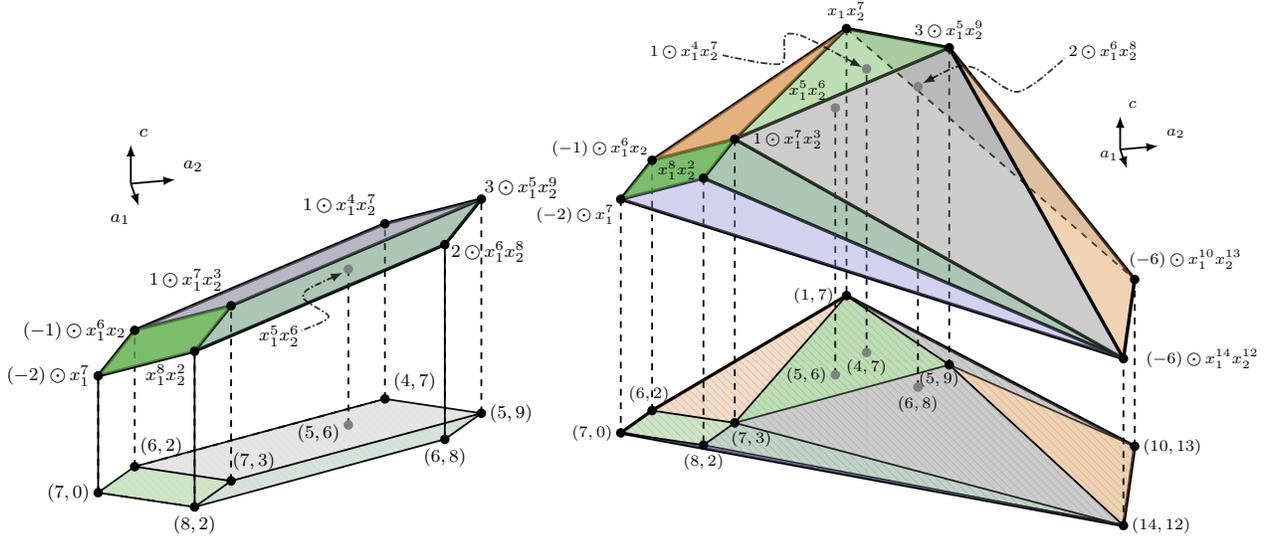
\begin{figure}
	\begin{adjustbox}{width=\textwidth}
	\centering
		\tdplotsetmaincoords{60}{80}
	\begin{tikzpicture}[thick,  scale = 0.7, tdplot_main_coords]
	\tikzset{font={\fontsize{8pt}{12}\selectfont}}
	\coordinate (O) at (0,0,0);
	% arrow type + heasd or tails
	%
	\def\arrA{latex-}
	\def\arrB{-latex}	
	\def\lw{1.5}
	% Upper envelop vertices
	%
	\def\a{1}
	\def\b{0.5}
	\def\ra{3}
	\def\s{4}
	\coordinate (A) at (7* \a, 3 * \a, 1 * \b + \s) ;
	\coordinate (B) at (8 * \a, 2 * \a, 0 * \b+ \s);
	\coordinate (C) at (5 * \a, 9 * \a, 3 * \b+ \s);
	\coordinate (D) at (6 * \a, 8 * \a, 2 * \b+ \s);	
	\coordinate (E) at (6 * \a, 1 * \a, -1 * \b+ \s);
	\coordinate (F) at (7 * \a, 0 * \a, -2 * \b+ \s);
	\coordinate (G) at (4 * \a, 7 * \a, 1 * \b+ \s);				
	\coordinate (H) at (5 * \a, 6 * \a, 0 * \b+ \s);

	\coordinate (a) at (7* \a, 3 * \a, 0 * \b);
	\coordinate (b) at (8 * \a, 2 * \a, 0 * \b);
	\coordinate (c) at (5 * \a, 9 * \a, 0 * \b);
	\coordinate (d) at (6 * \a, 8 * \a, 0 * \b);	
	\coordinate (e) at (6 * \a, 1 * \a, 0 * \b);
	\coordinate (f) at (7 * \a, 0 * \a, 0 * \b);
	\coordinate (g) at (4 * \a, 7 * \a, 0 * \b);				
	\coordinate (h) at (5 * \a, 6 * \a, 0 * \b);

	\draw (g) -- (c) -- (d) -- (b)  -- (f) -- (e) -- (g) -- cycle;	
	\fill [step=0.1cm, pattern color=gray, pattern = north west lines, opacity=0.25](g) -- (c) -- (d) -- (b)  -- (f) -- (e) -- (g) -- cycle;
	\fill[greet,opacity=0.2](a) -- (b) -- (d) -- (c)  -- cycle;
	\fill[greeo,opacity=0.3](a) -- (b) -- (f) -- (e)  -- cycle;
	\fill[gray,opacity=0.2](a) -- (e) -- (g) -- (c)  -- cycle;
	
	\draw[dashed] (a) -- (A);
	\draw[dashed] (b) -- (B);
	\draw[dashed] (c) -- (C);
	\draw[dashed] (d) -- (D);
		
	\draw[dashed] (e) -- (E);
	\draw[dashed] (f) -- (F);
	\draw[dashed] (g) -- (G);
	\draw[dashed] (h) -- (H);

	\draw[line width=\lw] (A) -- (C);
	\draw[line width=\lw] (A) -- (B);
	\draw[line width=\lw] (A) -- (E);
	\draw[line width=\lw] (E) -- (F);
	\draw[line width=\lw] (F) -- (B);
	\draw[line width=\lw] (C) -- (D);
	\draw[line width=\lw] (B) -- (D);
	\draw[] (B) -- (b);
	\draw[] (F) -- (f);
	\draw[] (D) -- (d);
	\draw[line width=\lw] (C) -- (G);
	\draw[line width=\lw] (E) -- (G);

\fill[lav,opacity=0.3](A) -- (C) -- (G) -- (E)  -- cycle;
%\fill[gray,opacity=0.4](G) -- (C) -- (c) -- (g)  -- cycle;
%\fill[gray,opacity=0.5](E) -- (G) -- (g) -- (e)  -- cycle;
%\fill[gray,opacity=0.4](F) -- (E) -- (e) -- (f)  -- cycle;
\fill[gray,opacity=0.4](A) -- (E) -- (G) -- (C)  -- cycle;
\fill[greeo,opacity=0.8](A) -- (B) -- (F) -- (E)  -- cycle;
\fill[greet,opacity=0.4](A) -- (B) -- (D) -- (C)  -- cycle;
%\fill[cof,opacity=0.3](C) -- (D) -- (d) -- (c)  -- cycle;
%\fill[cof,opacity=0.3](f) -- (F) -- (B) -- (b)  -- cycle;
	
\draw[] (a) -- (c);
\draw[] (a) -- (b);
\draw[] (a) -- (e);

\draw[dashed] (b) -- (f);

\draw[dashed] (e) -- (g);

	\begin{scope}
	\fill (a) circle[radius = \ra pt] node [above right, xshift = -0.7ex, yshift = 0.2ex] {$(7, 3)$};		
	\fill (b) circle[radius = \ra pt] node [below] {$(8, 2)$};			
	\fill (c) circle[radius = \ra pt] node [right] {$(5, 9)$};
	\fill (d) circle[radius = \ra pt] node [below] {$(6, 8)$};
	\fill (e) circle[radius = \ra pt] node [above right, yshift=0.3ex, xshift = -0.3ex] {$(6, 2)$};
	\fill (f) circle[radius = \ra pt] node [left] {$(7, 0)$};
	\fill (g) circle[radius = \ra pt] node [above right] {$(4, 7)$};
	\fill[gray] (h) circle[radius = \ra pt];
	\fill (h) node [left, yshift = -0.8ex] {$(5, 6)$};
	\end{scope}
	
	\begin{scope}
	\fill (A) circle[radius = \ra pt] node [above left, yshift = 0.7ex] {$1\odot x_1^7x_2^3$};	
	\fill (B) circle[radius = \ra pt] node [below left, yshift =-0.5ex] {$x_1^8x_2^2$};
	\fill (C) circle[radius = \ra pt] node [above right, yshift =-0.7ex, xshift=-0.5ex] {$3\odot x_1^5x_2^9$};
	\fill (D) circle[radius = \ra pt] node [right, yshift=-0.7ex, xshift = -0.2ex] {$2\odot x_1^6x_2^8$};
	\fill (E) circle[radius = \ra pt] node [left] {$(-1)\odot x_1^6x_2$};	
	\fill (F) circle[radius = \ra pt] node [left] {$(-2) \odot x_1^7$};
	\fill (G) circle[radius = \ra pt] node [above left] {$1\odot x_1^4x_2^7$};
	\fill[gray] (H) circle[radius = \ra pt];
	\fill (H) node [below, yshift = -4.7ex, xshift = -6.8ex] {$x_1^5 x_2^6$};
	\end{scope}
	
	\draw[opacity=0.8, densely dashdotted] [\arrA] plot  [smooth] coordinates {(5.13 * \a, 5.85 * \a, 0 * \b+ \s)
		 (6 * \a, 4.8 * \a, 0 * \b+ \s)
	 (7 * \a, 5.2 * \a, 0 * \b+ \s) (7.5 * \a, 4.5 * \a, 0 * \b+ \s)
	};
		% small axis plot
	%
	\begin{scope}[shift={(0,0.5, 0)}]
	\def\scalaxis{1}
	\coordinate (axis_o) at (-3*\a, 2 , 2);
	\coordinate (axis_a1) at (-3*\a + \scalaxis, 2 , 2);
	\coordinate (axis_a2) at (-3*\a, 2+ \scalaxis , 2);
	\coordinate (axis_c) at (-3*\a, 2 , 2 + \scalaxis);
	\end{scope}
	
	\draw[\arrB] (axis_o) -- (axis_c) node [above right] {$c$};
	\draw[\arrB] (axis_o) -- (axis_a1) node [below left] {$a_1$};
	\draw[\arrB] (axis_o) -- (axis_a2) node [above right] {$a_2$};	
	\end{tikzpicture}
	\hspace*{-7ex}
	% f(2)
%
\tdplotsetmaincoords{60}{80}
\begin{tikzpicture}[thick, scale = 0.6,  tdplot_main_coords]
\tikzset{font={\fontsize{7pt}{12}\selectfont}}
\coordinate (O) at (0,0,0);
% arrow type + heasd or tails
%
\def\arrA{latex-}
\def\arrB{-latex}	

% Upper envelop vertices
%
\def\a{1}
\def\b{0.5}
\def\ra{3.5}
\def\s{8}
\def\lw{1.5}
\coordinate (A) at (7* \a, 3 * \a, 1 * \b + \s) ;
\coordinate (B) at (8 * \a, 2 * \a, 0 * \b+ \s);
\coordinate (C) at (5 * \a, 9 * \a, 3 * \b+ \s);
\coordinate (D) at (6 * \a, 8 * \a, 2 * \b+ \s);	
\coordinate (E) at (6 * \a, 1 * \a, -1 * \b+ \s);
\coordinate (F) at (7 * \a, 0 * \a, -2 * \b+ \s);
\coordinate (G) at (4 * \a, 7 * \a, 1 * \b+ \s);				
\coordinate (H) at (5 * \a, 6 * \a, 0 * \b+ \s);

\coordinate (I) at (10* \a, 13 * \a, -6 * \b +\s);
\coordinate (J) at (14 * \a, 12 * \a, -6 * \b +\s);
\coordinate (K) at (1 * \a, 7 * \a, 0 +\s);
\coordinate (L) at (5 * \a, 6 * \a, 0 +\s);

\coordinate (a) at (7* \a, 3 * \a, 0 * \b);
\coordinate (b) at (8 * \a, 2 * \a, 0 * \b);
\coordinate (c) at (5 * \a, 9 * \a, 0 * \b);
\coordinate (d) at (6 * \a, 8 * \a, 0 * \b);	
\coordinate (e) at (6 * \a, 1 * \a, 0 * \b);
\coordinate (f) at (7 * \a, 0 * \a, 0 * \b);
\coordinate (g) at (4 * \a, 7 * \a, 0 * \b);				
\coordinate (h) at (5 * \a, 6 * \a, 0 * \b);

\coordinate (i) at (10* \a, 13 * \a, 0);
\coordinate (j) at (14 * \a, 12 * \a, 0);
\coordinate (k) at (1 * \a, 7 * \a, 0);
\coordinate (l) at (5 * \a, 6 * \a, 0);	

\fill[gray, opacity=0.25](K) -- (C) -- (I) -- cycle;	

\draw[line width=\lw] (f) -- (e) -- (k) -- (i) -- (j)--  cycle;	
\fill [step=0.1cm, pattern color=gray, pattern = north west lines, opacity=0.5](f) -- (e) -- (k) --  (i) -- (j)  --  cycle;	

\draw[dashed] (K) -- (I);

\draw[] (c) -- (i);	
\draw[] (a) -- (c);
\draw[] (b) -- (j); 
\draw[] (a) -- (j); 

\fill[greeo,opacity=0.3](f) -- (e) -- (a) -- (b)  -- cycle;	
\fill[cof,opacity=0.3](k) -- (a) -- (e) -- cycle;
\fill[lav ,opacity=0.6](f) -- (b) -- (j) -- cycle;
\fill[greet,opacity=0.3](b) -- (a) -- (j)  -- cycle;
\fill[cof, opacity=0.4](c) -- (i) -- (j) -- cycle;
\fill[gray, opacity=0.4](k) -- (c) -- (i) -- cycle;	
\fill[greeo,opacity=0.4](k) -- (a) -- (c) -- cycle;
\fill[gray,opacity=0.4](a) -- (c) -- (j) -- cycle;	

\draw[line width=\lw] (A) -- (K);
\draw[] (a) -- (k);
\draw[line width=\lw] (C) -- (K);
\draw[] (c) -- (k);
\draw[line width=\lw] (E) -- (K);
\draw[] (e) -- (k);

\draw[dashed] (h) -- (H); 
\draw[dashed] (g) -- (G); 
%	\draw[dashed] (l) -- (L); 
\draw[dashed] (d) -- (D); 

\draw[line width=\lw] (E) -- (A) -- (C);
\draw[line width=\lw] (F) -- (E) -- (A) -- (B) -- cycle;
\draw[] (f) -- (e) -- (a) -- (b) -- cycle;

\draw[dashed] (K) -- (k);
\draw[dashed] (C) -- (c);
\draw[dashed] (A) -- (a);
\draw[dashed] (E) -- (e);
\draw[dashed] (F) -- (f);
\draw[dashed] (B) -- (b);
\draw[dashed] (i) -- (I);

\draw[line width=\lw] (C) -- (I);
\draw[line width=\lw] (B) -- (J); 
\draw[line width=\lw] (A) -- (J); 
\draw[line width=\lw] (F) -- (J); 

\fill[greeo,opacity=0.8](F) -- (E) -- (A) -- (B)  -- cycle;	
\fill[cof,opacity=0.7](K) -- (A) -- (E) -- cycle;

\fill[greet,opacity=0.4](B) -- (A) -- (J)  -- cycle;	
\fill[lav ,opacity=0.4](F) -- (B) -- (J) -- cycle;	
\draw[dashed] (j) -- (J); 
\fill[cof, opacity=0.4](C) -- (I) -- (J) -- cycle;	
\fill[gray,opacity=0.4](A) -- (C) -- (J) -- cycle;	
\draw[line width=\lw] (I) -- (J); 
\draw[line width=\lw] (C) -- (J); 
\draw[] (c) -- (j); 
\fill[greeo,opacity=0.4](K) -- (A) -- (C) -- cycle;

\begin{scope}
\fill (a) circle[radius = \ra pt] node [below right, xshift=-1.2ex] {$(7, 3)$};		
\fill (b) circle[radius = \ra pt] node [below] {$(8, 2)$};			
\fill (c) circle[radius = \ra pt] node [below, xshift =-1ex, yshift =0.2ex] {$(5, 9)$};
\fill[gray] (d) circle[radius = \ra pt] ;
\fill[] (d) node [below] {$(6, 8)$};
\fill (e) circle[radius = \ra pt] node [above, xshift=-0.3ex] {$(6, 2)$};
\fill (f) circle[radius = \ra pt] node [left] {$(7, 0)$};
\fill[gray] (g) circle[radius = \ra pt];
\fill[] (g)  node [below, yshift=0.2ex] {$(4, 7)$};
\fill[gray] (h) circle[radius = \ra pt];
\fill[] (h) node [left] {$(5, 6)$};

\fill (i)  circle[radius = \ra pt] node [right]  {$(10, 13)$};
\fill (j)  circle[radius = \ra pt] node [right]  {$(14, 12)$};
\fill (k)  circle[radius = \ra pt] node [ left]  {$(1, 7)~$};
%	\fill (l)  circle[radius = \ra pt] node [above left]  {$(5, 6)$};

\end{scope}

\begin{scope}
\fill (A) circle[radius = \ra pt] node [right, xshift =0.95ex] {$1\odot x_1^7x_2^3$};	
\fill (B) circle[radius = \ra pt] node [left, yshift =0.8ex, xshift=0.1ex] {$x_1^8x_2^2$};
\fill (C) circle[radius = \ra pt] node [above] {$3\odot x_1^5x_2^9$};
\fill[gray] (D) circle[radius = \ra pt];
\fill[] (D) node [right, xshift =  14ex, yshift  = 3.5ex] {$2 \odot x_1^6x_2^8$};

\draw[opacity=0.8, densely dashdotted] [\arrA] plot  [smooth] coordinates { (5.85 * \a, 8.15 * \a, 2 * \b+ \s)
	(3.5 * \a, 9.8 * \a, 0 * \b+ \s)
	(5 * \a, 10.8 * \a, 0 * \b+ \s)
	(1 * 3.5 * \a, 1 *  12.3 * \a, 0 * \b+ \s)
%	(7 * \a, 5.2 * \a, 0 * \b+ \s) (7.5 * \a, 4.5 * \a, 0 * \b+ \s)
};

\fill (E) circle[radius = \ra pt] node [above left, yshift=-0.5ex, xshift =0.5ex] {$(-1)\odot x_1^6x_2$};	
\fill (F) circle[radius = \ra pt] node [yshift=-1.5ex, xshift=-4.5ex] {$(-2)\odot x_1^7$};
\fill[gray] (G) circle[radius = \ra pt];
\fill[] (G) node [above left, xshift = -13.7ex]  {$1\odot x_1^4x_2^7$} ;

\draw[opacity=0.8, densely dashdotted] [\arrA] plot  [smooth] coordinates {(3.9 * \a, 6.9 * \a, 1 * \b+ \s)
	(1 * \a, 5.5 * \a, 0 * \b+ \s)
	(2.5 * \a, 5 * \a, 0 * \b+ \s)
	(1.2 *  1.5 * \a, 1.2 * 3 * \a, 0 * \b+ \s)
};

\fill[gray] (H) circle[radius = \ra pt];
\fill[] (H) node [above left] {$x_1^5x_2^6$};

\fill (I) circle[radius = \ra pt] node [yshift=3ex, xshift=5ex, yshift=-1ex ] {$(-6)\odot x_1^{10}x_2^{13}$};
\fill (J) circle[radius = \ra pt] node [right, xshift=1ex] {$~(-6)\odot x_1^{14}x_2^{12}$};
\fill (K) circle[radius = \ra pt] node [above] {$x_1x_2^7$};
%	\fill (L) circle[radius = \ra pt] node [above left] {$(5, 6, 0)$};
\end{scope}

% small axis plot
%
\begin{scope}[shift={(0,0.5, 0)}]
\def\scalaxis{1}
\coordinate (axis_o) at (+5*\a, 13*\a , 6);
\coordinate (axis_a1) at (+5*\a +\scalaxis, 13*\a , 6);
\coordinate (axis_a2) at (+5*\a, 13+ \scalaxis , 6);
\coordinate (axis_c) at (+5*\a, 13*\a  ,  6 +\scalaxis);
\end{scope}

\draw[\arrB] (axis_o) -- (axis_c) node [above right] {$c$};
\draw[\arrB] (axis_o) -- (axis_a1) node [above left] {$a_1$};
\draw[\arrB] (axis_o) -- (axis_a2) node [above right] {$a_2$};	
\end{tikzpicture}
\end{adjustbox}
\vspace*{-4ex}
\caption{Left: The polytope associated with $h^{(2)}$ and its dual subdivision.
         Right: $\mathcal{P}(f^{(2)})$ and dual subdivision of $f^{(2)}$.
         In both figures, dual subdivisions have been translated along the $-c$ direction (downwards) and separated from the polytopes for visibility.}
\label{fig:eg-h2-f2}
\end{figure}

%Note that we sometimes say ``dual subdivision of $f$'' for ``dual subdivision of the Newton polygon of $f$'' like in the captions of these figures.

\section{Proofs}\label{sec:suppl-proofs}

\subsection{Proof of Corollary~\ref{cor:num-vert-on-uf}}
\begin{proof}
Let $V_1$ and $V_2$ be the sets of vertices on the upper and  lower envelopes of $P$ respectively.
By Theorem~\ref{thm:minkowski-face-bound}, $P$ has 
\[
n_1 \coloneqq 2 \sum_{j=0}^{d} \binom{m-1}{j}
\]
vertices in total. By construction, we have $| V_1 \cup V_2 | = n_1$. 
It is well-known that zonotopes are centrally symmetric and so there are equal number of vertices on the upper and lower envelopes, i.e., $|V_1| = |V_2|$.
Let $P' \coloneqq \pi(P)$ be the projection of $P$ into $\real^d$.  
Since the projected vertices are assumed to be in general positions, $P'$ must be a $d$-dimensional zonotope generated by $m$ nonparallel line segments.
Hence, by Theorem~\ref{thm:minkowski-face-bound} again, $P'$ has 
\[
n_2 \coloneqq 2 \sum_{j=0}^{d-1} \binom{m-1}{j}
\]
vertices.
For any vertex $v \in P$, $\pi(v)$ is a vertex of $P'$ if and only if $v$ belongs to both the upper and lower envelopes, i.e., $v \in V_1 \cap V_2$.
Therefore the number of vertices on $P'$ equals $|V_1 \cap V_2|$.  
By construction, we have $|V_1 \cap V_2| = n_2$.
Consequently the number of vertices on the upper envelope is
\[
|V_1| = \frac{1}{2}(|V_1 \cup V_2| - |V_1 \cap V_2|) + |V_1 \cap V_2|
      = \frac{1}{2}(n_1 - n_2) + n_2
      = \sum_{j=0}^{d} \binom{m}{j}.  \qedhere
\]
\end{proof}

\subsection{Proof of Proposition~\ref{prop:representation2}}
\label{proof:representation2}
\begin{proof}
Writing $A = A_+ - A_-$, we have
\begin{align*}
\rho^{(l+1)}({x}) &= \bigl( A_+ - A_- \bigr) \bigl( F^{(l)}(x) - G^{(l)}(x) \bigr) + b \\
              &= \bigl( A_+ F^{(l)}(x)+A_- G^{(l)}(x) + b \bigr) - \bigl( A_+ G^{(l)}(x) + A_- F^{(l)}(x) \bigr)  \\
              &= H^{(l+1)}(x) - G^{(l+1)}(x), \\
\nu^{(l+1)}(x) &= \max \bigl\{ \rho^{(l+1)}(y) , \, t \bigr\} \\
&= \max \bigl\{ H^{(l+1)}(x) - G^{(l+1)}(x),\, t\bigr\}\\
&= \max \bigl\{ H^{(l+1)}(x) ,\, G^{(l+1)}(x) + t\bigr\} - G^{(l+1)}(x) \\
&= F^{(l+1)}(x) - G^{(l+1)}(x). \qedhere
\end{align*}
\end{proof}

\subsection{Proof of Theorem~\ref{thm:RTF-is-FFNNReLU}}
\begin{proof}
It remains to establish the ``only if'' part. 
We will write $\sigma_t (x)\coloneqq  \max \{ x, t \}$.
Any tropical monomial $b_i x^{\alpha_i}$ is clearly such a neural network as
\[
		b_i x^{\alpha_i} = (\sigma_{-\infty} \circ \rho_i)(x) = \max \{ \alpha_i^\tp x + b_i  , -\infty \}.
\]
If two tropical polynomials $p$ and $q$ are represented as neural networks with $l_p$ and $l_q$ layers respectively,
		\begin{align*}
		p(x) &= \bigl( \sigma_{-\infty} \circ \rho^{(l_p)}_p \circ \sigma_0 \circ \dots \sigma_0 \circ \rho^{(1)}_p \bigr) (x), \\
		q(x) &= \bigl( \sigma_{-\infty} \circ \rho^{(l_q)}_q \circ \sigma_0 \circ \dots \sigma_0 \circ \rho^{(1)}_q \bigr) (x),
		\end{align*}
		then $(p \oplus q)(x) =  \max \{ p(x), q(x) \}$ can also be written as a neural network with $\max \{ l_p , \l_q \}+1$ layers:
		\[
		(p \oplus q)(x) = \sigma_{-\infty} \bigl( [ \sigma_0 \circ \rho_{1} ]( y(x) )
		+ [ \sigma_0 \circ \rho_{2} ]( y(x) )
		- [ \sigma_0 \circ \rho_{3} ]( y(x) ) \bigr),
		\]
		where $y : \real^d \to \real^2$ is given by $y(x) = (p(x) , q(x))$ and $\rho_{i} : \real^2 \to \real$, $i=1,2,3$, are linear functions defined by
		\[
		\rho_1( y ) = y_1 - y_2, \quad
		\rho_2( y ) = y_2, \quad
		\rho_3( y ) = -y_2.
		\]
		Thus, by induction, any tropical polynomial can be written as a neural network with ReLU activation.
		Observe also that if a tropical polynomial is the tropical sum of $r$ monomials, then it can be written as a neural network with no more than $\lceil \log_2 r \rceil+1$ layers.

		Next we consider a tropical rational function $(p\oslash q)(x) = p(x) - q(x)$ where $p$ and $q$ are tropical polynomials. Under the same assumptions, we can represent $p\oslash q$ as
\[
(p\oslash q)(x) = \sigma_{-\infty} \bigl( [\sigma_0 \circ \rho_{4}](y(x)) - [ \sigma_0 \circ \rho_{5} ](y(x)) 
	  + [\sigma_0 \circ \rho_{6}](y(x)) - [\sigma_0 \circ \rho_{7}](y(x)) \bigr)
\]
		where $\rho_{i} : \real^2 \to \real^2$, $i=4,5,6,7$, are linear functions defined by
		\[
		\rho_{4}(y) = y_1, \quad
		\rho_{5}(y) = -y_1, \quad
		\rho_{6}(y) = -y_2, \quad
		\rho_{7}(y) = y_2.
		\]
Therefore $p \oslash q$ is also a neural network  with at most $\max \{ l_p , \l_q \}+1$ layers.

Finally, if $f$ and $g$ are tropical polynomials that are respectively tropical sums of $r_f$ and $r_g$ monomials, then the discussions above show that $(f\oslash g) (x) = f(x) - g(x)$ is a neural network with at most $\max \{ \lceil \log_2 r_f \rceil , \, \lceil \log_2 r_g \rceil \}  + 2$ layers.
\end{proof}

\subsection{Proof of Proposition~\ref{prop:CPWL-is-RTF}}
\begin{proof}
It remains to establish the ``if'' part. 
Let $\mathbb{R}^d$ be divided into $N$ polyhedral region on each of which $\nu$ restricts to a linear function
\[
\ell_i(x) =  a_i^\tp x  + b_i, \quad a_i \in \mathbb{Z}^d, \quad b_i \in \real, \quad i=1, \dots , L,
\] 
i.e., for any $x \in \real^d$, $\nu(x) = \ell_i(x)$ for some $i \in \{1, \dots, L \}$.
It follows from \cite{tarela1999region} that we can find $N$ subsets of $\{ 1 , \dots, L \}$, denoted by $S_j$, $j=1, \dots , N$, so that $\nu$ has a representation
	\begin{align*}
	\nu (x) = \max_{j = 1, \dots , N} \min_{ i \in S_j } \ell_i.
	\end{align*}
It is clear that each $\ell_i$ is a tropical rational function.
Now for any tropical rational functions  $p$ and $q$,
		\[
		\min \{ p , q \} = - \max \{ -p , -q \} 
		= 0 \oslash [ (0 \oslash p) \oplus ( 0 \oslash q ) ]
		= [p \odot q] \oslash [ p \oplus q ].
		\]
		Since $p \odot q$ and $p \oplus q$ are both tropical rational functions, so is their tropical quotient. By induction, $\min_{i \in S_j} \ell_i$ is a tropical rational function for any $j=1,\dots,N$, and therefore so is their tropical sum $\nu$.
	\end{proof}

\subsection{Proof of Proposition~\ref{prop:tropsig}}
\begin{proof}
For a one-layer neural network  $\nu(x) =\max\{Ax +b, t\} = (\nu_1(x),\dots, \nu_p(x))$ with $A \in \mathbb{R}^{p \times d}$, $b \in \mathbb{R}^{p}$, $x\in \mathbb{R}^d$, $t \in (\mathbb{R}\cup \{-\infty\})^{p}$, we have
\[
\nu_k(x) =\biggl( b_k \odot  \bigodot_{j=1}^d x_j^{a_{kj}}\biggr)\oplus t_k
=\biggl( b_k \odot  \bigodot_{j=1}^d x_j^{a_{kj}}\biggr)\oplus \biggl(
t_k \odot \bigodot_{j=1}^d x_j^{0}\biggr), \qquad k =1,\dots,p.
\]
So for any $k=1,\dots,p$, if we write $\bar{b}_1 = b_k$, $\bar{b}_2 = t_k$, $\bar{a}_{1j} = a_{kj}$, $\bar{a}_{2j}= 0$, $j = 1, \dots, d$, then
\[
\nu_k(x) = \bigoplus_{i=1}^2 \bar{b}_i \bigodot_{j=1}^{d} x_j^{\bar{a}_{ij}}
\] 
is clearly a tropical signomial function. Therefore $\nu$ is a tropical signomial map. The result for arbitrary number of layers then follows from using the same recurrence as in the proof in Section~\ref{proof:representation2}, except that now the entries in the weight matrix are allowed to take real values, and the maps $H^{(l)}(x)$, $G^{(l)}(x)$, $F^{(l)}(x)$ are tropical signomial maps.  Hence every layer can be written as a tropical rational signomial map $\nu^{(l)} = F^{(l)}\oslash G^{(l)}$.
\end{proof}

\subsection{Proof of Proposition~\ref{prop:db}}

We prove a slightly more general result.
\begin{proposition}[Level sets]\label{prop:boundaries}
	Let $f \oslash g \in \RL(d,1) = \mathbb{T}(x_1,\dots,x_d)$. 
	\begin{enumerate}[\upshape (i), topsep=0ex, itemsep=0ex]
		\item\label{region1} Given a constant $c>0$, the \emph{level set}
		\[
		\mathcal{B} \coloneqq  \{ x \in \real^d : f(x) \oslash g(x) = c \}
		\]
	divides $\real^d$ into at most $\R(f)$ connected polyhedral regions where $f(x) \oslash g(x) > c$, and at most $\R(g)$ such regions where $f(x) \oslash g(x) <c$. 
	\item \label{region2}
	 If $c \in \real$ is such that there is no tropical monomial in $f(x)$ that differs from any tropical monomial in $g(x)$ by $c$, then the level set $\mathcal{B}$ is contained in a tropical hypersurface,
	\[
	\mathcal{B} \subseteq \mathcal{T} ( \max \{ f(x), \, g(x)+ c \} )
	=\mathcal{T} ( c \odot g \oplus f ).
	\]
	\end{enumerate}
\end{proposition}
\begin{proof} 
 We show that the bounds on the numbers of connected positive (i.e., above $c$) and negative (i.e., below $c$) regions   are as we claimed in \ref{region1}. The tropical hypersurface of $f$ divides $\real^d$ into $\R(f)$ convex regions $C_1, \dots, C_{\R(f)}$ such that $f$ is linear on each $C_i$. As $g$ is piecewise linear and convex over $\real^d$, $f\oslash g = f-g$ is piecewise linear and concave on each $C_i$.
Since the level set $\{x : f(x) - g(x) = c\}$ 
and the superlevel set $\{ x : f(x) - g(x) \geq c \}$ must be convex by the concavity of $f-g$, there is at most one positive region in each $C_i$.
Therefore the total number of connected positive regions cannot exceed $\R(f)$.
Likewise, the tropical hypersurface of $g$ divides $\real^d$ into $\R(g)$ convex regions on each of which
$f \oslash g$ is convex. The same argument shows that the number of connected negative regions does not exceed $\R(g)$.

We next address \ref{region2}. Upon rearranging terms, the level set becomes
\[
\mathcal{B} = \bigl\{ x \in \mathbb{R}^d :f(x) = g(x) +c \bigr\}.
\]		
Since $f(x)$ and $g(x)+c$ are both tropical polynomial, we  have 
\begin{align*}
f(x) &= b_1x^{\alpha_1}\oplus \dots \oplus b_rx^{\alpha_r}, \\
g(x)+c &= c_1 x^{\beta_1}\oplus \dots \oplus c_sx^{\beta_s},
\end{align*}
with  appropriate multiindices $\alpha_1,\dots,\alpha_r$, $\beta_1,\dots,\beta_s$, and real coefficients $b_1,\dots,b_r$, $c_1,\dots,c_s$.
By the assumption on the monomials, we have that $x_0 \in \mathcal{B}$ only if there exist $i,j$ so that $\alpha_i \neq \beta_j$ and  $b_i x_0^{\alpha_i}=  c_jx_0^{\beta_j}$.
This completes the proof since if we combine the monomials of $f(x)$ and $g(x)+c$ by (tropical) summing them into a single tropical polynomial,  $\max \{ f(x), \, g(x)+c \}$, the above implies that on the level set, the value of the combined tropical polynomial is attained by at least two monomials and therefore $x_0 \in \mathcal{T} ( \max\{ f(x), \, g(x)+ c \} )$.
\end{proof}
Proposition~\ref{prop:db} follows immediately  from Proposition~\ref{prop:boundaries} since the decision boundary  $\{x\in \mathbb{R}^d : \nu(x) = s^{-1}(c)\}$ is a level set of the tropical rational function $\nu$.

\subsection{Proof of Theorem~\ref{thm:main-bound}}
\label{prof:main-bound}

The linear regions of a tropical polynomial map $F \in \PL(d, m)$ are all convex but this is not necessarily the case for a  tropical rational map $F \in \RL(d,n)$.
Take for example a bivariate real-valued function $f(x,y)$ whose graph in $\real^3$ is a  pyramid with base $\{ (x , y) \in \real^2 : x,y \in [-1,1] \}$ and zero everywhere else,
then the linear region where $f$ vanishes is $\real^2 \setminus \{ (x , y) \in \real^2 : x,y \in [-1,1] \}$, which is nonconvex. The nonconvexity invalidates certain geometric arguments that only apply in  the convex setting. Nevertheless there is a way to subdivide each of the nonconvex linear regions into convex ones to get ourselves back into the convex setting. We will start with the number of \emph{convex} linear regions for tropical rational maps although later we will deduce the required results for the number of linear regions (without imposing convexity).

We first extend the notion of tropical hypersurface to tropical rational maps:
Given a tropical rational map $F \in \RL(d,m)$, we define $\mathcal{T}(F)$ to be the boundaries between adjacent linear regions. 
When $F = (f_1,\dots,f_m) \in \PL(d,m)$, i.e., a tropical polynomial map, this set is exactly the union of tropical hypersurfaces $\mathcal{T}(f_i)$, $i=1, \dots, m$. 
Therefore this definition of $\mathcal{T}(F)$ extends Definition~\ref{def:trophype}.

For a tropical rational map $F$, we will examine the smallest number of convex regions that form a refinement of $\mathcal{T}(F)$.
For brevity, we  will call this  the \emph{convex degree} of $F$; for consistency, the number of linear regions of $F$ we will call its \emph{linear degree}. We define convex degree formally below.  We will write $\restr{F}{C}$ to mean the restriction of map $F$ to $C \subseteq \real^d$.
\begin{definition}\label{def:aff_rest} 
The \emph{convex degree} of a tropical rational map $F \in \RL(d, n)$ is the minimum division of $\mathbb{R}^d$ into convex regions over which $F$ is linear, i.e.
	\[
	\pR(F)\coloneqq	 \min \bigl\{n :   C_1\cup \dots \cup C_n = \mathbb{R}^d,\; C_i\; \text{convex,}\; \restr{F}{C_i} \; \text{linear}\bigr\}.
	\]
Note that $C_1, \dots, C_{\pR(F)}$ either divide $\mathbb{R}^d$ into the same regions as $\mathcal{T}(F)$ or form a refinement.

	For $m \leq d$, we will denote by $\pR(F \mid m)$ the maximum convex degree obtained by restricting $F$ to an $m$-dimensional affine subspace in $\real^d$,
	i.e.,
	\[
	\pR(F \mid m) \coloneqq \max \bigl\{ \pR( F\vert_{\Omega} ) : \Omega \subseteq \real^d\; \text{is an $m$-dimensional affine space}\bigr\}.
	\]
\end{definition}

For any $F \in \RL(d, n)$, there is at least one tropical polynomial map that subdivides $\mathcal{T}(F)$, and so convex degree is well-defined (e.g., if $F = (p_1 \oslash q_1,\dots, p_n \oslash q_n)\in \RL(d, n)$, then we may choose $P = (p_1, \dots, p_n, q_1, \dots, q_n)\in \PL(d, 2n)$). Since the linear regions of a tropical polynomial map are always convex, we have $\R(F) = \pR(F)$ for any $F \in \PL(d,n)$. 

Let $F =(f_1,\dots,f_n) \in \RL(d,n)$  and $\alpha = (a_1, \dots, a_n) \in \mathbb{Z}^n$. Consider the tropical rational function\footnote{This is in the sense of a tropical power but we stay consistent to our slight abuse of notation and write $F^\alpha$ instead of $F^{\odot \alpha}$.}
\[
F^\alpha \coloneqq \alpha^\tp F = a_1 f_1 +\dots + a_n f_n =  \bigodot_{j=1}^n f_j^{a_j} \in \RL(d, 1).
\]
For some $\alpha$, $F^\alpha$ may have fewer linear regions than $F$, e.g, $\alpha = (0,\dots, 0)$. As such, we need the following notion.
\begin{definition}
$\alpha = (a_1, \dots, a_n) \in \mathbb{Z}^n$ is said to be a \emph{general exponent} of $F \in \RL(d,n)$ if the linear regions of $F^\alpha$ and the linear regions of $F$ are identical. 
\end{definition}
We show that  general exponent always exists  for any $F \in \RL(d,n)$  and may be chosen to have all entries nonnegative.
\begin{lemma}
\label{lemma:general_exp}
Let $F \in \RL(d, n)$. Then 
	\begin{enumerate}[\upshape (i), topsep=0ex, itemsep=0ex]
		\item\label{exp1} $\R(F^\alpha) = \R(F)$  if and only if $\alpha$ is a general exponent;
		\item\label{exp2} $F$ has a general exponent $\alpha \in \mathbb{N}^n$.
	\end{enumerate} 
\end{lemma}
\begin{proof}
It follows from the definition of tropical hypersuface that $\mathcal{T}(F^\alpha)$ and $\mathcal{T}(F)$ comprise respectively the points $x\in \real^d$ at which $F^\alpha$ and $F$ are not differentiable.  
Hence $\mathcal{T}(F^\alpha) \subseteq \mathcal{T}(F)$, which implies that $\R(F^\alpha) < \R(F)$ unless $\mathcal{T}(F^\alpha) = \mathcal{T}(F)$.
This concludes \ref{exp1}.  

For \ref{exp2}, we need to show that there always exists an $\alpha \in \mathbb{N}^n$ such that $F^\alpha$ divides its domain $\real^d$ into the same set of linear regions as $F$.
In other words, for every pair of adjacent linear regions of $F$, the $(d-1)$-dimensional face in $\mathcal{T}(F)$ that separates them is also present in $\mathcal{T}(F^\alpha)$ and so $\mathcal{T}(F^\alpha) \supseteq \mathcal{T}(F)$.

Let $L$ and $M$ be adjacent linear regions of $F$.  The differentials of $\restr{F}{L}$ and $\restr{F}{M}$ must have integer coordinates, i.e., $\restr{dF}{L}, \restr{dF}{M}\in \mathbb{Z}^{n \times d}$.
Since $L$ and $M$ are distinct linear regions, we must have $\restr{dF}{L}\neq \restr{dF}{M}$ (or otherwise $L$ and $M$ can be merged into a single linear region).
Note that the differentials of $\restr{F^\alpha}{L}$ and $\restr{F^\alpha}{M}$ are given by $\alpha^\tp \restr{dF}{L}$ and $\alpha^\tp \restr{dF}{M}$.

To ensure the $(d-1)$-dimensional face separating $L$ and $M$ still exists in $\mathcal{T}(F^\alpha)$, we need to choose $\alpha$ so that $\alpha^\tp \restr{dF}{L} \neq \alpha^\tp \restr{dF}{M}$. Observe that the solution to $(\restr{dF}{L} - \restr{dF}{M})^\tp \alpha = 0$ is contained in a one-dimensional subspace of $\real^n$.

Let $\mathcal{A}(F)$ be the collection of all pairs of adjacent linear regions of $F$.  Since the set of $\alpha$ that degenerates two adjacent linear regions into a single one, i.e.,
\[
\mathcal{S} \coloneqq \bigcup_{(L, M) \in \mathcal{A}(F)} \bigl\{ \alpha \in \mathbb{N}^n :  (\restr{dF}{L} - \restr{dF}{M})^\tp \alpha = 0)\bigr\},
\]
is contained in a union of a finite number of hyperplanes in $\real^n$, $\mathcal{S}$ cannot cover the entire lattice of nonnegative integers $\mathbb{N}^n$.
Therefore the set $\mathbb{N}^n \cap ( \real^n \setminus \mathcal{S} )$ is nonempty and any of its element is a general exponent for $F$.
\end{proof}
Lemma~\ref{lemma:general_exp} shows that we may study the linear degree of a tropical rational map by studying that of  a tropical rational function, for which the results in Section~\ref{sec:transform-tropical-poly} apply.

We are now ready to prove a key result on the convex degree of composition of tropical rational maps.
\begin{theorem}\label{theorem:comp_transform}
	Let $F = (f_1,\dots,f_m) \in \RL(n, m)$ and  $G \in \RL(d,n)$. Define $H = (h_1,\dots,h_m) \in\RL(d,m)$ by
	\[
	h_i\coloneqq f_i \circ G , \qquad i = 1, \dots, m.
	\]
	Then
	\[
	\R(H) \leq \pR(H) \leq \pR(F \mid d)
	\cdot \pR(G).
	\]
\end{theorem}
\begin{proof}
	Only the upper bound requires a proof. Let $k =\pR(G)$. By the definition of $\pR(G)$, there exist convex sets $C_1,\dots,C_k \subseteq \mathbb{R}^d$ whose union is $\mathbb{R}^d$ and on each of which $G$ is linear. So $\restr{G}{C_i}$ is some affine function $\rho_i$. For any $i$,
	\[
	\pR(F \circ \rho_i)\leq \pR(F \mid d),
	\]
by the definition of $\pR(F \mid d)$. Since $F \circ G = F \circ \rho_i$ on $C_i$, we have 
\begin{gather*}
	\pR(F \circ G) \leq \sum_{i=1}^{k}\pR(F\circ \rho_i).
\shortintertext{Hence}
	\pR(F \circ G) \leq \sum_{i=1}^{k}\pR(F\circ \rho_i) \leq \sum_{i=1}^{k} \pR(F \mid d) = \pR(F \mid d) \cdot \pR(G). \qedhere
\end{gather*}
\end{proof}

We now apply our observations on tropical rational functions to neural networks. The next lemma follows directly from Corollary~\ref{cor:num-vert-on-uf}.
\begin{lemma}\label{lemma:layer_tranform} Let  $\sigma^{(l)} \circ \rho^{(l)} : \mathbb{R}^{n_{l-1}} \to \mathbb{R}^{n_l}$ where $ \sigma^{(l)}$ and $\rho^{(l)}$ are the affine transformation and activation of the $l$th layer of a neural network. If $d \leq n_{l}$, then
	\[
	\pR(\sigma^{(l)} \circ \rho^{(l)}  \mid d) \leq \sum_{i=0}^{d} \binom{n_{l}}{i}.
	\]
\end{lemma}

\begin{proof}
$\pR(\sigma^{(l)} \circ \rho^{(l)}  \mid d)$ is the maximum convex degree of  a tropical rational map $F =(f_1,\dots,f_{n_l}) : \real^d \to \real^{n_l}$ of the form
\[
f_i(x) \coloneqq \sigma^{(l)}_i \circ \rho^{(l)} \circ (b_1 \odot x^{\alpha_1},\dots, b_{n_{l-1}} \odot x^{\alpha_{n_{l-1}}}),\qquad i = 1,\dots, n_l.
\]
For a general affine transformation $\rho^{(l)} $,
\[
\rho^{(l)} (b_1 \odot x^{\alpha_1},\dots, b_{n_{l-1}} \odot x^{\alpha_{n_{l-1}}}) = \bigl(b'_1 \odot x^{\alpha'_1},\dots, b'_{n_{l}} \odot x^{\alpha'_{n_{l}}}\bigr) \eqqcolon G(x)
\]
for some $\alpha'_1,\dots,\alpha'_{n_l}$ and $b'_1,\dots,b'_{n_l}$, and we denote this map by $G : \real^d \to \real^{n_l}$. So $f_i =  \sigma_i^{(l)} \circ G$. By Theorem~\ref{theorem:comp_transform},
we have $\pR(\sigma^{(l)} \circ \rho^{(l)}  \mid d) = 	\pR(\sigma^{(l)} \mid d) \cdot \pR(G) = \pR(\sigma^{(l)} \mid d)$; note that $\pR(G) = 1$ as $G$ is a linear function.

We have thus reduced the problem to determining a bound on the convex degree of a single layer neural network with $n_l$ nodes  $\nu = (\nu_1,\dots,\nu_{n_l}) : \mathbb{R}^d \to \mathbb{R}^{n_l}$. 
Let $\gamma = (c_1, \dots, c_{n_l}) \in \mathbb{N}^{n_l}$ be a nonnegative general exponent for $\nu$. Note that
\begin{align*}
\bigodot\limits_{j=1}^{n_l} \nu_j^{c_j} =
\bigodot\limits_{j=1}^{n_l} \biggl[\biggl(\bigodot\limits_{i=1}^{d} b_i \odot x^{a^{+}_{ji}} \biggr) 
\oplus
\biggl( \bigodot\limits_{i=1}^{d} x^{a^{-}_{ji}} \biggr)\odot  t_j \biggr]^{c_j}
- 
\bigodot\limits_{j=1}^{n_l}\biggl(\bigodot\limits_{i=1}^{d} x^{a^{-}_{ji}} \biggr)^{c_j}.
\end{align*}
Since the last term is linear in $x$, we may drop it without affecting the convex degree of the entire expression. It remains to determine an upper bound for the number of linear regions of the tropical polynomial 
	\[
	h(x) = \bigodot\limits_{j=1}^{n_l}\biggl[	\biggl(\bigodot\limits_{i=1}^{d} b_i \odot x^{a^{+}_{ji}} \biggr) 
	\oplus \biggl( \bigodot\limits_{i=1}^{d} x^{a^{-}_{ji}} \biggr)\odot  t_j \biggr]^{c_j},
	\]
which we will obtain by counting vertices of the polytope $\cP(h)$.  By Propositions~\ref{prop:polytope-exp} and \ref{prop:polytope-ops} the polytope $\cP(h)$ is given by a weighted Minkowski sum
\begin{gather*}
	\sum_{j=1}^{n_l} c_j \cP\biggl[\biggl(\bigodot\limits_{i=1}^{d} b_i \odot x^{a^{+}_{ji}} \biggr) 
	\oplus	\biggl(	\bigodot\limits_{i=1}^{d} x^{a^{-}_{ji}} \biggr)\odot  t_j \biggr]. 
\shortintertext{By Proposition~\ref{prop:polytope-ops} again,}
	\cP\biggl[\biggl(\bigodot\limits_{i=1}^{d} b_i \odot x^{a^{+}_{ji}}	\biggr) 
	\oplus	\biggl(	\bigodot\limits_{i=1}^{d} x^{a^{-}_{ji}} \biggr)\odot  t_j \biggr]
=\operatorname{Conv} \bigl( \cV ( \cP(f) ) \cup \cV ( \cP(g) ) \bigr)
\shortintertext{where}
	 f(x) = \bigodot\limits_{i=1}^{d} b_i \odot x^{a^{+}_{ji}} \qquad
	 \text{and} \qquad
     g(x) =  \biggl(\bigodot\limits_{i=1}^{d} x^{a^{-}_{ji}} \biggr)  \odot  t_j
\end{gather*}
are tropical monomials.  Therefore $\cP(f)$, $\cP(g)$ are just points in $\mathbb{R}^{d+1}$ and  $\operatorname{Conv} \bigl( \cV ( \cP(f) ) \cup \cV ( \cP(g) ) \bigr)$ is a line in $\mathbb{R}^{d+1}$.
Hence $\cP(h)$ is a Minkowski sum of $n_{l}$ line segments in $\mathbb{R}^{d+1}$, i.e., a zonotope, and Corollary~\ref{cor:num-vert-on-uf} completes the proof.
\end{proof}

Using Lemma~\ref{lemma:layer_tranform}, we obtain a bound on the number of linear regions created by one layer of a neural network.
\begin{theorem}\label{thm:main}
Let $\nu : \mathbb{R}^d \to \mathbb{R}^{n_L}$ be an $L$-layer neural network satisfying assumptions \ref{ass1}--\ref{ass3} with  $F^{(l)}$,  $G^{(l)}$,$H^{(l)}$, and $\nu^{(l)}$ as defined in Proposition~\ref{prop:representation2}. Let $n_{l} \geq d$ for all $l=1,\dots,L$. Then
\[
\pR(\nu^{(1)})	= \R(G^{(1)}) = \R(H^{(1)}) = 1, \qquad
\pR(\nu^{(l+1)}) \leq \pR(\nu^{(l)}) \cdot  \sum_{i=0}^{d} \binom{n_{l+1}}{i}.
\]
\end{theorem}
\begin{proof}
The $l=1$ case  follows from the fact that $G^{(1)}(x) = A^{(1)}_{-}x$ and $H^{(1)}(x) = A^{(1)}_{+}x +b^{(1)}$ are both linear, which in turn forces $\pR(\nu^{(1)}) = 1$ as in the proof of Lemma~\ref{lemma:layer_tranform}. 
Since $\nu^{(l)} = (\sigma^{(l)}\circ \rho^{(l)}) \circ \nu^{(l-1)}$, the recursive bound follows from Theorem~\ref{theorem:comp_transform} and Lemma~\ref{lemma:layer_tranform}.
\end{proof}

Theorem~\ref{thm:main-bound} follows from applying Theorem~\ref{thm:main} recursively.


\ifdefined\includeSuppl
\else

\bibliography{reference}
\bibliographystyle{icml2018}
\fi

}
{}

%\appendix
%\section{Do \emph{not} have an appendix here}
%
%\textbf{\emph{Do not put content after the references.}}
%%
%Put anything that you might normally include after the references in a separate
%supplementary file.
%
%We recommend that you build supplementary material in a separate document.
%If you must create one PDF and cut it up, please be careful to use a tool that
%doesn't alter the margins, and that doesn't aggressively rewrite the PDF file.
%pdftk usually works fine. 
%
%\textbf{Please do not use Apple's preview to cut off supplementary material.} In
%previous years it has altered margins, and created headaches at the camera-ready
%stage. 
%%%%%%%%%%%%%%%%%%%%%%%%%%%%%%%%%%%%%%%%%%%%%%%%%%%%%%%%%%%%%%%%%%%%%%%%%%%%%%%
%%%%%%%%%%%%%%%%%%%%%%%%%%%%%%%%%%%%%%%%%%%%%%%%%%%%%%%%%%%%%%%%%%%%%%%%%%%%%%%

\end{document}